%% file: main.tex
\documentclass[10pt,journal,compsoc]{IEEEtran}
\ifCLASSOPTIONcompsoc
  \usepackage[nocompress]{cite}
\else
  \usepackage{cite}
\fi

\usepackage{epsfig}
\usepackage{graphicx,subfigure}
\usepackage{epstopdf}
\usepackage{amsmath}
\usepackage{amssymb}
\usepackage{color}
\usepackage{booktabs}
\usepackage{multirow}
\usepackage[table]{xcolor}
\usepackage{url}
\usepackage[ruled, linesnumbered]{algorithm2e}
\usepackage{xcolor}
\usepackage{enumitem}
\usepackage{amsthm}
\usepackage{empheq}
\usepackage{latexsym}
\usepackage{tabularx}
\usepackage{array}
\usepackage{diagbox}
\usepackage{listings}

\definecolor{commentcolor}{rgb}{0,0,0}
\newtheorem{theorem}{Theorem}

\newtheorem{proposition}{Proposition}

\def\A{\mathbf{A}}

\def\R{\mathbf{R}}

\def\Q{\mathbf{Q}}

\def\x{\mathbf{x}}

\def\Q{\mathbf{Q}}

\DeclareMathOperator*{\trace}{trace}

\DeclareMathOperator*{\rank}{rank}

\ifCLASSINFOpdf
\else
\fi
\begin{document}
%
\title{Affine Correspondences between Multi-Camera Systems for Relative Pose Estimation}
%
%
%
%

\author{Banglei Guan \ and \ Ji Zhao
\IEEEcompsocitemizethanks{\IEEEcompsocthanksitem B. Guan is with College of Aerospace Science and Engineering, National University of Defense Technology, Changsha 410073, China.\protect\\
E-mail: guanbanglei12@nudt.edu.cn
\IEEEcompsocthanksitem J. Zhao is in Beijing, China.
\protect\\
E-mail: zhaoji84@gmail.com
}
\thanks{(Corresponding author: Ji Zhao)}
}

%
%

\markboth{}
{Shell \MakeLowercase{\textit{et al.}}: Bare Demo of IEEEtran.cls for Computer Society Journals}
%



\IEEEtitleabstractindextext{%
\begin{abstract}
We present a novel method to compute the relative pose of multi-camera systems using two affine correspondences (ACs). Existing solutions to the multi-camera relative pose estimation are either restricted to special cases of motion, have too high computational complexity, or require too many point correspondences (PCs). Thus, these solvers impede an efficient or accurate relative pose estimation when applying RANSAC as a robust estimator. This paper shows that the 6DOF relative pose estimation problem using ACs permits a feasible minimal solution, when exploiting the geometric constraints between ACs and multi-camera systems using a special parameterization. We present a problem formulation based on two ACs that encompass two common types of ACs across two views, \emph{i.e.}, inter-camera and intra-camera. Moreover, the framework for generating the minimal solvers can be extended to solve various relative pose estimation problems, \emph{e.g.}, 5DOF relative pose estimation with known rotation angle prior. Experiments on both virtual and real multi-camera systems prove that the proposed solvers are more efficient than the state-of-the-art algorithms, while resulting in a better relative pose accuracy. Source code is available at \url{https://github.com/jizhaox/relpose-mcs-depth}.
\end{abstract}

}

\maketitle

\IEEEdisplaynontitleabstractindextext

%
\IEEEpeerreviewmaketitle

\IEEEraisesectionheading{\section{Introduction}}
\IEEEPARstart{E}{stimating} the relative poses of a monocular camera, or a multi-camera system is a key problem in computer vision, which plays an important role in structure from motion (SfM), simultaneous localization and mapping (SLAM), and augmented reality (AR)~\cite{nister2004efficient,henrikstewenius2005solutions,scaramuzza2011visual,kneip2014efficient,schoenberger2016sfm,hane20173d,heng2019project}. A multi-camera system refers to a system of individual cameras that are rigidly fixed onto a single body, and it can cover a large field-of-view to obtain more information about the environment. Motivated by the fact that multi-camera systems are an interesting choice in the context of robotics applications such as autonomous drones and vehicles, relative pose estimation for multi-camera systems has started to receive attention lately~\cite{hane20173d,heng2019project,alyousefi2020multi,martyushev2020efficient,guanICCV2021minimal}.

Different from monocular cameras which are modeled by the perspective camera model, the multi-camera systems can be modeled by the generalized camera model~\cite{grossberg2001general,sturm2004generic,miraldo2011point}. The generalized camera model does not have a single center of projection. The light rays that pass through the multi-camera system do not intersect in a single center of projection, \emph{i.e.}, non-central projection~\cite{pless2003using}. 
Thus, the relative pose estimation problem of multi-camera systems~\cite{henrikstewenius2005solutions} is different from that of monocular cameras~\cite{nister2004efficient}, which leads to different equations. In addition, since feature correspondences established by feature matching inevitably contain outliers, the relative pose estimation algorithms are typically employed inside a robust estimation framework such as the Random Sample Consensus (RANSAC)~\cite{fischler1981random}. The computational complexity of the RANSAC estimator increases exponentially with respect to the number of feature correspondences needed. Thus, minimal solvers for relative pose estimation are very desirable for RANSAC schemes, which maximizes the probability of picking an all-inlier sample and reduces the number of necessary iterations~\cite{henrikstewenius2005solutions,li2008linear,kim2009motion,lim2010estimating,ventura2015efficient,kneip2016generalized}.
\begin{figure}[t]
	\begin{center}	
		\subfigure[Inter-camera ACs]
		{
			\includegraphics[width=0.42\linewidth]{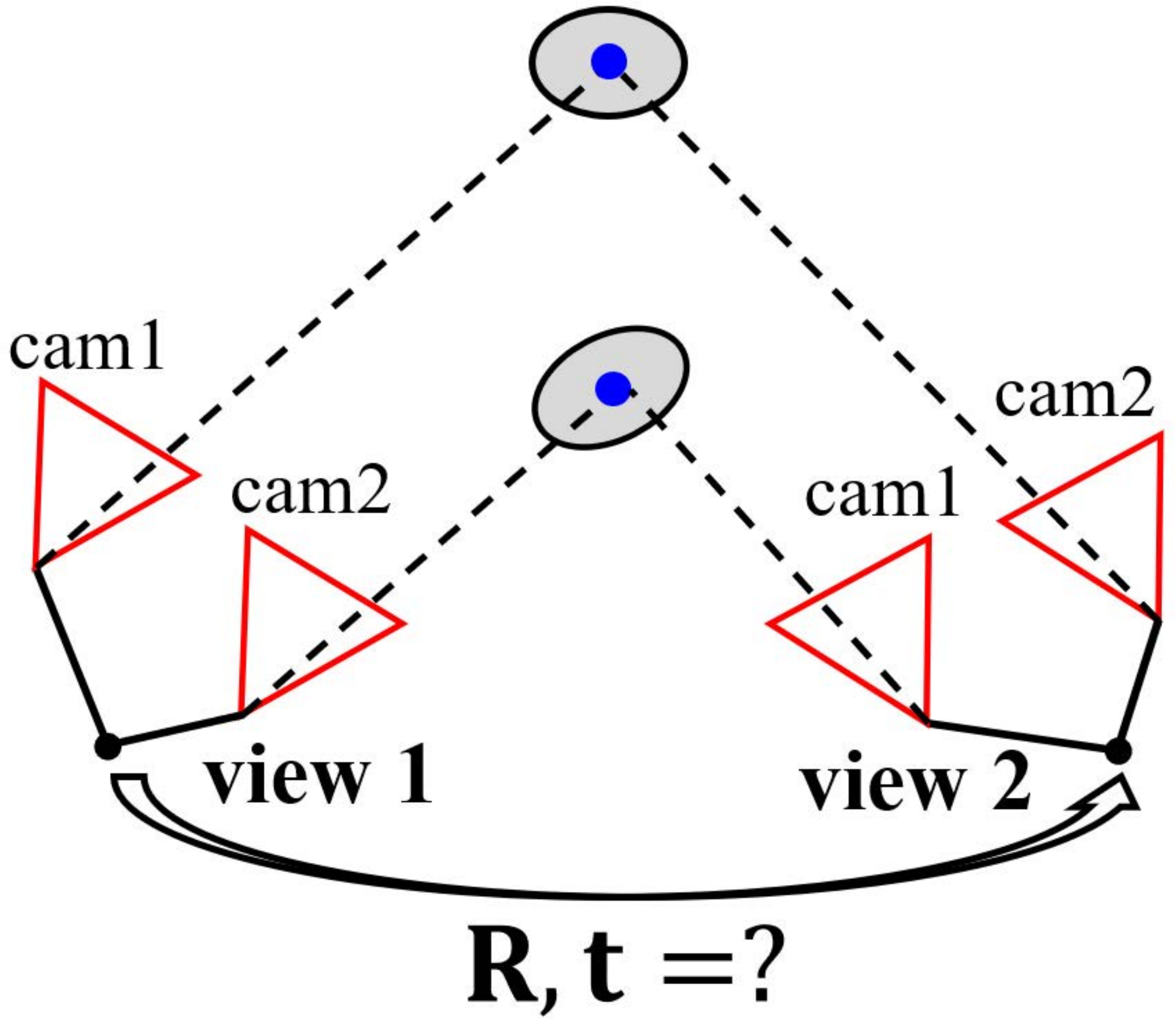}
		}
		\hspace{0.18in}
		\subfigure[Intra-camera ACs]
		{
			\includegraphics[width=0.42\linewidth]{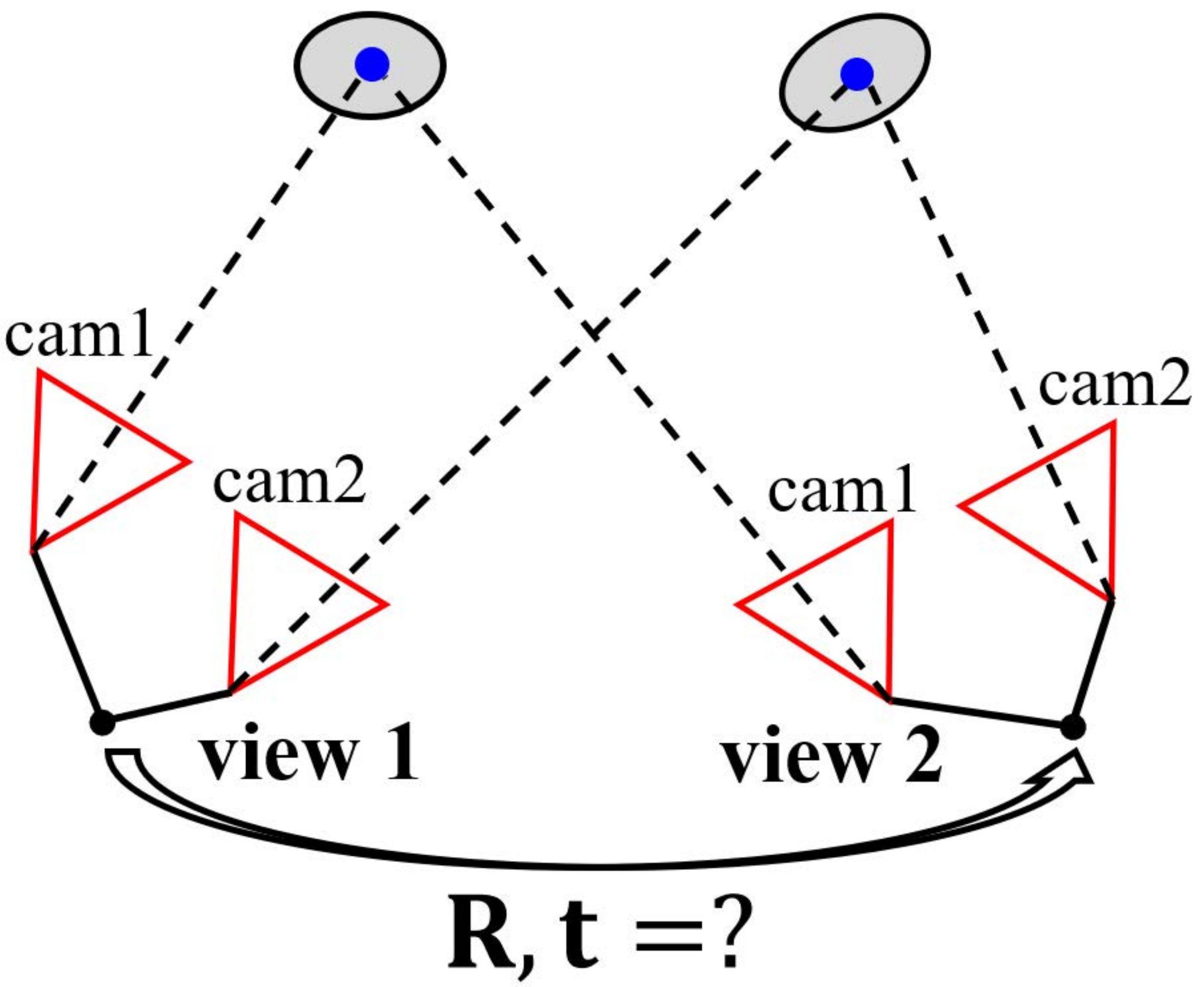}
		}
	\end{center}
	\vspace{-5pt} 
	\caption{Relative pose estimation from two ACs for a multi-camera system. Red triangle represents a single camera, and gray ellipse represents a spatial patch which relates to an AC. Specifically, inter-camera ACs refer to correspondences which are seen by different cameras over two consecutive views. Intra-camera ACs refer to correspondences which are seen by the same camera over two consecutive views.}
	\label{fig:taster_5DOFand6DOF}
\end{figure} 

The development of minimal solvers for 6DOF relative pose estimation of multi-camera systems can be traced back to the method of Stew{\'e}nius~\emph{et al.} with six point correspondences (PCs)~\cite{henrikstewenius2005solutions}. Later, some methods have been subsequently proposed, such as the linear method with seventeen PCs~\cite{li2008linear}, iterative optimization method~\cite{kneip2014efficient} and global optimization method~\cite{Zhao2020GEM}. In addition, a lot of methods exploit additional constraints to further reduce the number of feature correspondences needed~\cite{hee2013motion,hee2014relative,sweeney2014solving,sweeney2015computing,liu2017robust}. Typical scenarios for these additional constraints include the information from an IMU and the specific movement of the camera. For the mobile devices equipped with both the cameras and the IMU, there is a well-known invariance of relative rotation angle across different coordinate systems. It is shown in~\cite{LiHeng-423,martyushev2020efficient,Li2020Relative} that the relative rotation angle of the IMU can be directly used as the relative rotation angle of the cameras, without requiring the rotational alignment between two sensors. The relative rotation angle is stable and accurate in general. This advantage makes the fusion of camera and IMU more flexible and convenient. With the known relative rotation angle, the 5DOF relative pose of multi-camera systems can be solved by five PCs~\cite{martyushev2020efficient}.	

In recent years, a number of solvers use affine correspondences (ACs), instead of PCs, to estimate the relative pose, which also reduces the number of required correspondences~\cite{bentolila2014conic,raposo2016theory,hajder2019relative,barath2020making,Eichhardt2020Relative,alyousefi2020multi,guanICCV2021minimal}. Because an AC carries more information than a PC. However, existing AC-based solvers to the 6DOF relative pose estimation for multi-camera systems are either restricted to pose priors~\cite{Guan_ICRA2021} or require at least six ACs~\cite{alyousefi2020multi}. Moreover, to the best of our knowledge, there is no AC-based solver to solve the relative pose estimation for multi-camera systems with known rotation angle. Thus, it is desirable to find AC-based minimal solvers for the relative pose estimation of multi-camera systems, whose efficiency and accuracy are both satisfactory. This allows us to reduce the computational complexity of the RANSAC procedure.

In this paper, we focus on the relative pose estimation problem of multi-camera systems from a minimal number of two ACs, see Fig.~\ref{fig:taster_5DOFand6DOF}. This work is the extension of our previous conference paper~\cite{guan2022relative_eccv}. We furthermore present the relative pose estimation for multi-camera systems with a known relative rotation angle. Moreover, this work includes discussion of degenerated configurations, more detailed analysis, additional comparisons, and real-world experiments. We propose minimal solvers based on the common configurations of two ACs across two views, \emph{i.e.}, inter-camera and intra-camera. The contributions of this paper are:
\begin{itemize}
		\item We derive the geometric constraints between ACs and multi-camera systems using a special parameterization, which eliminates the translation parameters by utilizing two depth parameters. Moreover, the implicit constraints about the affine transformation constraints are found and proved.
		\item We develop two novel minimal solvers for 6DOF relative pose estimation of multi-camera systems from two ACs. Both solvers are designed for two common types of ACs in practice. We obtain practical solvers for totally new settings. In addition, three degenerate cases are proved.
		\item With a known relative rotation angle, two minimal solvers are first proposed for 5DOF relative pose estimation of multi-camera systems. The proposed solvers are convenient and flexible in practical applications when the cameras and an IMU are equipped. The relative rotation angle of the cameras can be provided by the IMU directly, even though camera$-$IMU extrinsics are unavailable or unreliable.
		\item We exploit a unified and versatile framework for generating the minimal solvers, which uses the hidden variable technique to eliminate the depth parameters. This framework can be extended to solve various relative pose estimation problems, \emph{e.g.}, relative pose estimation with known rotation angle prior. 
\end{itemize}

The remainder of this paper is organized as follows: Section~\ref{sec:related_work} reviews related literature on relative pose estimation of multi-camera systems. Section~\ref{sec:6DOFmotion} presents the minimal solvers for 6DOF relative pose estimation of multi-camera systems. In Section~\ref{sec:5DOFmotion}, the proposed framework for generating the minimal solvers is extended to solve 5DOF relative pose estimation of multi-camera systems. Section~\ref{sec:configurations} shows the proofs of degenerated configurations for the proposed solvers. The performance of the proposed solvers is evaluated on both virtual and real multi-camera systems in Section~\ref{sec:experiments}. Finally, concluding remarks are given in Section~\ref{sec:conclusion}.

\section{Related Work}
\label{sec:related_work}
Stew{\'e}nius~\emph{et al.} proposed the first minimal solver based on algebraic geometry, and this solver requires 6 PCs in order to come up with 64 solutions~\cite{henrikstewenius2005solutions}. Kim~\emph{et al.} later presented alternative solvers for relative pose estimation with non-overlapping multi-camera systems using second-order cone programming~\cite{kim2007visual} or branch-and-bound technique over the space of all rotations~\cite{kim2009motion}. Clipp~\emph{et al.} also derived a solver using 6 PCs for non-overlapping multi-camera systems~\cite{clipp2008robust}. Lim~\emph{et al.} presented antipodal epipolar constraints on the relative pose by exploiting the geometry of antipodal points, which are available in large field-of-view cameras~\cite{lim2010estimating}. Li~\emph{et al.} used 17 PCs to solve the relative pose of multi-camera systems linearly, which ignores side-constraints on the generalized essential matrix and the contained essential and rotation matrices~\cite{li2008linear}. Kneip and Li proposed an iterative approach for the relative pose estimation with an efficient eigenvalue minimization strategy~\cite{kneip2014efficient}. The above mentioned works are designed for 6DOF relative pose estimation of multi-camera systems.

A number of methods estimate the relative pose of multi-camera systems with a prior. Typically, the priors include available IMU information~\cite{hee2014relative,sweeney2014solving,liu2017robust,martyushev2020efficient} and multi-camera movement prior~\cite{hee2013motion,guanICCV2021minimal}, which reduce the DOF of the relative pose problem. For the devices equipped with both the cameras and the IMU, the relative rotation angle of multi-camera systems can be derived from the reading of an IMU without requiring the camera-IMU extrinsics. By making use of the known relative rotation angle provided by the IMU, Martyushev~\emph{et al.}~\cite{martyushev2020efficient} estimated the 5DOF relative pose from 5 PCs. With known camera$-$IMU extrinsics, the IMU measurements can provide vertical direction of the multi-camera system. Sweeney~\emph{et al.}~\cite{sweeney2014solving}, Lee~\emph{et al.}~\cite{hee2014relative} and Liu~\emph{et al.}~\cite{liu2017robust} proposed several minimal solvers with 4 PCs to solve 4DOF relative pose. When the multi-camera system is mounted on ground robots and the movement follows the Ackermann motion model, Lee~\emph{et al.}~\cite{hee2013motion} used a minimum of 2 PCs to recover the 2DOF relative pose.

Recently, using ACs to estimate the relative pose of multi-camera systems has drawn much attention. Alyousefi and Ventura~\cite{alyousefi2020multi} proposed a linear solver to recover the 6DOF relative pose using 6 ACs, which generalizes the 17 PCs solver proposed by Li~\emph{et al.}~\cite{li2008linear}. Guan~\emph{et al.}~\cite{Guan_ICRA2021} used a first-order approximation to relative rotation to estimate the 6DOF relative pose, which generalizes the 6 PCs solver proposed by Ventura~\emph{et al.}~\cite{ventura2015efficient}. They assume that the relative rotation of the multi-camera systems between two consecutive views is small. Furthermore, Guan~\emph{et al.}~\cite{guanICCV2021minimal} estimated the 3DOF relative pose under planar motion with a single AC and estimated the 4DOF relative pose with known vertical direction with 2 ACs. In this paper, we focus on using a minimal number of 2 ACs to estimate the 6DOF relative pose of multi-camera systems, which does not rely on any motion constraints or pose priors. Moreover, the solver generation procedure is extended to solve 5DOF relative pose estimation with known rotation angle prior.

\section{\label{sec:6DOFmotion}6DOF Relative Pose Estimation for Multi-Camera Systems}
In this section, we assume that both the intrinsic and extrinsic parameters of multi-camera systems are known. Aiming at the common configurations of two ACs across two views in Fig.~\ref{fig:taster_5DOFand6DOF}, our purpose is to develop the minimal solvers for 6DOF relative pose estimation using inter-camera ACs and intra-camera ACs. The proposed solvers are the most common ones in practice for multi-camera systems. In 6DOF relative pose estimation for multi-camera systems, the inputs for the proposed solvers are two ACs.

\subsection{\label{sec:Parameterization}Parameterization}
We first formulate and parameterize the relative pose estimation problem for a multi-camera system. As shown in Fig.~\ref{fig:solver_6dof}, the multi-camera system is composed of multiple perspective cameras. Denote the extrinsic parameters of individual perspective camera $C_i$ as $\{\mathbf{Q}_i, \mathbf{s}_i\}$, where $\mathbf{Q}_i$ and $\mathbf{s}_i$ represent relative rotation and translation to the reference of the multi-camera system, respectively. Denote the relative pose of multi-camera systems as $\{\mathbf{R}, \mathbf{t}\}$, which represents the relative rotation and translation from view 1 to view 2 of the multi-camera system.
\begin{figure}[tbp]
	\begin{center}
		\includegraphics[width=0.8\columnwidth]{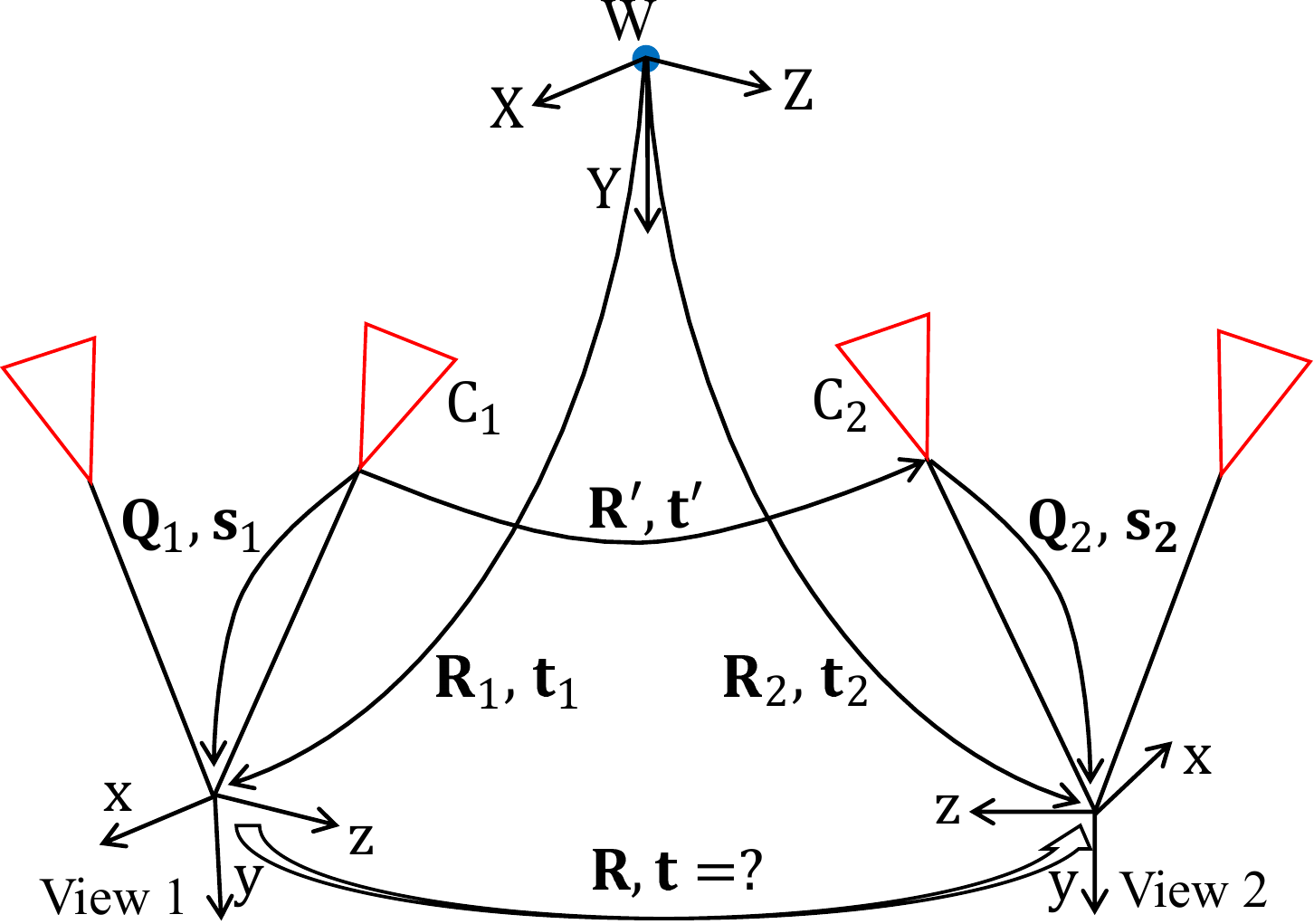}
	\end{center}
	\caption{Relative pose estimation for multi-camera systems.}
	\label{fig:solver_6dof}
\end{figure}

In order to eliminate the translation parameters, we use a special parameterization to formulate the relative pose estimation problem~\cite{henrikstewenius2005solutions}. Take an AC seen by the different cameras for an example. The $j\text{-th}$ AC relates the camera $C_1$ and $C_2$ across two views, see Fig.~\ref{fig:solver_6dof}. Let us denote the $j\text{-th}$ AC as $(\x_j, \x'_j, \A_j)$, where $\x_j$ and $\x'_j$ are the normalized homogeneous image coordinates of feature points in the view 1 and view 2, respectively. $\A_j$ is a $2\times 2$ local affine transformation, which relates the infinitesimal patches around $\x_j$ and $\x'_j$~\cite{raposo2016theory,barath2018efficient}. Suppose the $j\text{-th}$ AC is chosen to define a world reference system $W$. The origin of $W$ is set to the position of the $j\text{-th}$ AC in 3D space, and the orientation of $W$ is consistent with the reference of the multi-camera system in view~1. Denote the relative rotation and translation from reference $W$ to view~1 as $\{\mathbf{R}_1, \mathbf{t}_1\}$. Denote the relative rotation and translation from reference $W$ to view~2 as $\{\mathbf{R}_2, \mathbf{t}_2\}$. It can be seen that $\mathbf{R}_1 = \mathbf{I}$, $\mathbf{R}_2 = \R$. In this paper, the Cayley parameterization is used to parametrize the relative rotation $\mathbf{R}$:
{\begin{align}
		&\mathbf{R} = \frac{1}{1+q_x^2+q_y^2+q_z^2} \ . \nonumber \\ 
		&\begin{bmatrix}{1+q_x^2-q_y^2-q_z^2}& \ {2{q_x}{q_y}-2{q_z}}& \ {2{q_y}+2{q_x}{q_z}}\\
			{2{q_x}{q_y}+2{q_z}}& \ {1-q_x^2+q_y^2-q_z^2}& \ {2{q_y}{q_z}-2{q_x}}\\
			{2{q_x}{q_z}-2{q_y}}& \ {2{q_x}+2{q_y}{q_z}}& \ {1-q_x^2-q_y^2+q_z^2}
		\end{bmatrix},
		\label{eq:R6dof1}
\end{align}}

\noindent 
where $[1,q_x,q_y,q_z]^{\text{T}}$ is a homogeneous quaternion vector. Note that the Cayley parameterization introduces a degeneracy for $180^\circ$ rotations, but this is a rare case for consecutive views in the robotics applications~\cite{stewenius2005minimal,kneip2014efficient,zheng2015structure,zhao2020minimal}.

Next, we show that the translation parameters $\mathbf{t}_1$ and $\mathbf{t}_2$ can be removed by using two depth parameters. For a calibrated multi-camera system, each image point corresponds to a unique line in the reference of the multi-camera system. This line in 3D can be represented as a Pl{\"u}cker vector $\mathbf{L} = [\mathbf{p}^{\text{T}}, \mathbf{q}^{\text{T}}]^{\text{T}}$, where the 3D vectors $\mathbf{p}$ and $\mathbf{q}$ represent the unit direction vector and the moment vector, respectively~\cite{pless2003using}. They satisfy the constraint $\mathbf{p}\cdot{\mathbf{q}} = 0$. Thus, the set of points $\mathbf{X}(\lambda)$ on the 3D line can be parameterized as 
\begin{align}
	\mathbf{X}(\lambda) = \mathbf{q} \times \mathbf{p} + \lambda \mathbf{p}, \quad \lambda \in \mathbb{R}.
\end{align}

\noindent where $\lambda$ is the unknown depth parameter of 3D point. Since the origin of $W$ is set to the 3D position $\mathbf{X}_j$ corresponding to $j\text{-th}$ AC, the Pl{\"u}cker coordinates of the line connecting the 3D position $\mathbf{X}_j$ and the optical center of camera $C_i$ can be described as $[\mathbf{p}_{ij}^{\text{T}}, \mathbf{q}_{ij}^{\text{T}}]^{\text{T}}$ in the reference of the multi-camera system. The 3D position $\mathbf{X}_j$ in view $k$ satisfies the following constraint:
\begin{align}
	\mathbf{q}_{ij} \times \mathbf{p}_{ij} + \lambda_{jk} \mathbf{p}_{ij} = \mathbf{R}_k 
	\begin{bmatrix}
		0, \ 0, \ 0
	\end{bmatrix}^{\text{T}}
	+ \mathbf{t}_k, \quad k = 1,2.
	\label{eq:3Dpositionconstraint}
\end{align}

Based on Eq.~\eqref{eq:3Dpositionconstraint}, the translation $\mathbf{t}_k$ from $W$ to view~$k$ is parameterized as the linear expression in the unknown depth parameter $\lambda_{jk}$ 
\begin{align}
	\mathbf{t}_k = \mathbf{q}_{ij} \times \mathbf{p}_{ij} + \lambda_{jk} \mathbf{p}_{ij}, \quad k = 1,2.
	\label{eq:calcu_trans}
\end{align}

\noindent 
where $k$ represents the index of the views, $i$ represents the index of the cameras, and $j$ represents the index of the ACs. {It can be seen that $\lambda_{j1}$ and $\lambda_{j2}$ are the depth parameters of the origin of $W$ in views 1 and 2, respectively.}

Through the above special parameterization, the 6DOF relative pose of multi-camera systems can be described by five unknowns, which consist of three rotation parameters $\{q_x,q_y,q_z\}$ and two depth parameters $\{\lambda_{j1}, \lambda_{j2}\}$.  

\subsection{Geometric Constraints}
It has been shown in Fig.~\ref{fig:solver_6dof} that each AC relates two perspective cameras in view~1 and view~2. The relative pose between two cameras $[\mathbf{R}', \mathbf{t}']$ is determined by the composition of four transformations: (i) from one perspective camera to view~1, (ii) from view~1 to $W$, (iii) from $W$ to view~2, (iv) from view~2 to the other perspective camera. Among these four transformations, the part~(i) and (iv) are determined by known extrinsic parameters. In the part~(ii) and (iii), there are unknowns $\mathbf{R}$, $\mathbf{t}_1$ and $\mathbf{t}_2$ which are parameterized as $\{q_x, q_y, q_z, \lambda_{j1}, \lambda_{j2}\}$. Formally, the relative pose $[\mathbf{R}', \mathbf{t}']$ is represented as:
{\begin{align}
		\begin{bmatrix}
			\R' & \mathbf{t}' \\
			\mathbf{0} & 1
		\end{bmatrix}
		&= \begin{bmatrix}
			\Q_2 & \mathbf{s}_2 \\
			\mathbf{0} & 1
		\end{bmatrix}^{-1}
		\begin{bmatrix}
			\R & \mathbf{t}_2 \\
			\mathbf{0} & 1
		\end{bmatrix}
		\begin{bmatrix}
			\mathbf{I} & \mathbf{t}_1 \\
			\mathbf{0} & 1
		\end{bmatrix}^{-1}
		\begin{bmatrix}
			\Q_1 & \mathbf{s}_1 \\
			\mathbf{0} & 1
		\end{bmatrix} \nonumber \\
		&=  
		\begin{bmatrix}
			\Q_2^{\text{T}} \R \Q_1 & \Q_2^{\text{T}} (\R \mathbf{s}_1 - \R \mathbf{t}_1 + \mathbf{t}_2 - \mathbf{s}_2) \\
			\mathbf{0} & 1
		\end{bmatrix}.
\end{align}}

Once the relative pose $[\mathbf{R}', \mathbf{t}']$ between two perspective cameras for each AC is expressed, the essential matrix $\mathbf{E}' = [\mathbf{t}']_{\times} \mathbf{R}'$ can be represented as:
\begin{align}
	\mathbf{E}' &= \Q_2^{\text{T}} \left(\R [\mathbf{s}_1 - \mathbf{t}_1]_\times + [\mathbf{t}_2 - \mathbf{s}_2]_\times \R \right) \Q_1. 
	\label{eq:essential_matrix}
\end{align}

By substituting Eq.~\eqref{eq:calcu_trans} into Eq.~\eqref{eq:essential_matrix}, we obtain: 
\begin{align}
	\mathbf{E}' = & -\lambda_{j1} \Q_2^{\text{T}}\R [\mathbf{p}_{ij}]_\times\Q_1 + \lambda_{j2}\Q_2^{\text{T}}[\mathbf{p}'_{ij}]_\times \R\Q_1  \nonumber \\
	& + \Q_2^{\text{T}} \left(\R [\mathbf{s}_1 - \mathbf{q}_{ij} \times \mathbf{p}_{ij}]_\times + [\mathbf{q}'_{ij} \times \mathbf{p}'_{ij} - \mathbf{s}_2]_\times \R \right) \Q_1.
	\label{eq:essential_matrix_expand}
\end{align}

It can be verified that each entry in the essential matrix $\mathbf{E}'$ is linear with $\{\lambda_{j1}, \lambda_{j2}\}$. Generally speaking, one AC $(\x_j, \x'_j, \A_j)$ yields three independent constraints on the relative pose estimation of a multi-camera system, which consist of one epipolar constraint derived from PC $(\x_j, \x'_j)$ and two affine transformation constraints derived from local affine transformation $\A_j$. With known intrinsic camera parameters, the epipolar constraint of PC between view~1 and view~2 is given as follows~\cite{HartleyZisserman-472}:
\begin{align}
	\x_j'^{\text{T}} \mathbf{E}' \x_j = 0,
	\label{eq:constraint_epipolar_single}
\end{align}

The affine transformation constraints which describe the relationship of essential matrix $\mathbf{E}'$ and local affine transformation $\A_j$ are formulated as follows~\cite{raposo2016theory,barath2018efficient}:
\begin{align}
	(\mathbf{E}'^{\text{T}} \x_j')_{(1:2)} = -\A_j^{\text{T}} (\mathbf{E}' \x_j)_{(1:2)},
	\label{eq:constraint_affine_single}
\end{align}

\noindent where the subscript $\text{(1:2)}$ represents the first two equations.

Even though the perspective cameras are assumed, the geometric constraints can straightforwardly be generalized to generalized camera model as long as local image patches across views are obtained equivalently by arbitrary central camera models~\cite{barath2018efficient,eichhardt2018affine}. Based on Eqs.~\eqref{eq:constraint_epipolar_single} and~\eqref{eq:constraint_affine_single}, two ACs provide six independent constraints. Considering that the relative pose estimation problem of multi-cameras systems has 6DOF, the number of constraints is equal to the number of unknowns. Thus, we explore the minimal solvers for 6DOF relative pose estimation using two ACs.

\subsection{Equation System Construction}
Note that the special parameterization has been adopted by choosing one AC as the origin of world reference system in the subsection~\ref{sec:Parameterization}, we found the PC derived from the chosen AC cannot contribute one constraint since the coefficients of the resulting equation are zero. Thus, when $j\text{-th}$ AC is chosen to build up the world reference system $W$, five equations can be provided by two ACs. Specifically, $j\text{-th}$ AC provides two equations based on Eq.~\eqref{eq:constraint_affine_single} and the other AC provides three equations based on Eqs.~\eqref{eq:constraint_epipolar_single} and~\eqref{eq:constraint_affine_single}. By substituting Eq.~\eqref{eq:essential_matrix_expand} into Eqs.~\eqref{eq:constraint_epipolar_single} and~\eqref{eq:constraint_affine_single} and using the hidden variable technique~\cite{cox2006using}, the five equations provided by two ACs can be written as:
\begin{align}
	\underbrace{\mathbf{F}_j(q_x, q_y, q_z)}_{5\times 3}
	\begin{bmatrix}
		\lambda_{j1} \\ \lambda_{j2} \\ 1
	\end{bmatrix}
	= \mathbf{0}.
	\label{eq:new_6dof_solver}
\end{align}

The entries in $\mathbf{F}_j$ are quadratic in unknowns $q_x$, $q_y$, and $q_z$. Since Eq.~\eqref{eq:new_6dof_solver} has non-trivial solutions, the rank of $\mathbf{F}_j$ satisfies $\rank(\mathbf{F}_j) \le 2$. Thus, all the $3\times 3$ sub-determinants of $\mathbf{F}_j$ must be zero. This gives 10 equations about three unknowns $\{q_x, q_y, q_z\}$. Moreover, we can choose the other AC to build up the world reference system, and its orientation is also consistent with the reference of the multi-camera system in view $1$. Suppose the $j'$-th AC is chosen, we build a new equation system about the same rotation parameters $\{q_x, q_y, q_z\}$, which is similar to Eq.~\eqref{eq:new_6dof_solver}:
\begin{align}
	\underbrace{\mathbf{F}_{j'}(q_x, q_y, q_z)}_{5\times 3}
	\begin{bmatrix}
		\lambda_{{j'}1} \\ \lambda_{{j'}2} \\ 1
	\end{bmatrix}
	= \mathbf{0}.
	\label{eq:new_6dof_solver2}
\end{align}

Note that Eq.~\eqref{eq:new_6dof_solver2} provides new constraints which is different from Eq.~\eqref{eq:new_6dof_solver}. We use the computer algebra system \texttt{Macaulay~2}~\cite{grayson2002macaulay} to find that there are one dimensional families of extraneous roots if only Eq.~\eqref{eq:new_6dof_solver} or Eq.~\eqref{eq:new_6dof_solver2} is used. This phenomenon has also been observed in~\cite{henrikstewenius2005solutions,martyushev2020efficient}. Based on Eqs.~\eqref{eq:new_6dof_solver} and~\eqref{eq:new_6dof_solver2}, we have 20 equations with three unknowns $\{q_x, q_y, q_z\}$:
{\begin{gather}
		\det(\mathbf{N}(q_x, q_y, q_z)) = 0, \label{eq:submatrix_3by3} \\
		\small{\mathbf{N}\in\{3\times 3 \text{ submatrices of } {\mathbf{F}_j}\} \cup 
			\{3\times 3 \text{ submatrices of } {\mathbf{F}_{j'}}\} \nonumber}.
\end{gather}}

\noindent 
These equations have a degree of $6$, {\emph{i.e.}, the highest of the degrees of the monomials with non-zero coefficients is 6.} 

Moreover, we derive extra implicit constraints in our problem, \emph{i.e.}, the rank of $(\mathbf{F}_j)_\text{(1:2,1:3)}$ is $1$. The proof is provided as follows:
\begin{theorem}
	\label{theorem:extra_constraint}
	When $j\text{-th}$ AC is chosen to build up the world reference system, the corresponding affine transformation constraints satisfy $\rank((\mathbf{F}_j)_{(1:2,1:3)}) = 1$.
\end{theorem}
\begin{proof}
	To achieve this goal, we need to prove that $(\mathbf{F}_j)_\text{(1:2,1:3)}$ has two linearly independent null space vectors $\mathbf{v}_1$ and $\mathbf{v}_2$. Based on Eq.~\eqref{eq:new_6dof_solver},   $\mathbf{v}_1=[\lambda_{j1},\lambda_{j2},1]^{\text{T}}$ is obviously a null space vector. Then we suppose that the second null space vector can be expressed as $\mathbf{v}_2=[\lambda_{z1},\lambda_{z2},0]^{\text{T}}$, where $\lambda_{z1}$ and $\lambda_{z2}$ are two unknown depth parameters of the origin of world reference system $W$ in camera 1 (view 1) and camera 2 (view 2), respectively.
	
	For the multi-camera system in Fig.~\ref{fig:solver_6dof}, we parameterize the transformation of cameras with respect to the world reference system $W$. Denote the transformation between camera 1 in view 1 and $W$ as $[\Q_1^{\text{T}}, \lambda_{z1}\Q_1^{\text{T}}{\mathbf{p}}_{ij}]$, and the transformation between camera 2 in view 2 and $W$ as $[\Q_2^{\text{T}}\mathbf{R}, \lambda_{z2}\Q_2^{\text{T}}{\mathbf{p}}'_{ij}]$. 
	The transformation between camera 1 in view 1 and camera 2 in view 2 $\{\tilde{\R},\tilde{\mathbf{t}}\}$ can be computed. Thus, the corresponding essential matrix $\tilde{\mathbf{E}}=[\tilde{\mathbf{t}}]_\times\tilde{\R}$ is represented as
	\begin{align} 
		\tilde{\mathbf{E}} &= -\lambda_{z1} \Q_2^{\text{T}}\R [\mathbf{p}_{ij}]_\times\Q_1 + \lambda_{z2}\Q_2^{\text{T}}[\mathbf{p}'_{ij}]_\times \R\Q_1.
		\label{eq:essential_matrix_single_generalized}
	\end{align}
		
	Note that the coefficients of the unknowns $\lambda_{z1}$ and $\lambda_{z2}$ in Eq.~\eqref{eq:essential_matrix_single_generalized} are the same as the coefficients of the unknowns $\lambda_{j1}$ and $\lambda_{j2}$ in Eq.~\eqref{eq:essential_matrix_expand}. Based on the Eq.~\eqref{eq:constraint_affine_single}, the affine transformation constraints can be written as follows
	\begin{align}
		(\mathbf{G}_j)_{(1:2,1:2)}
		\begin{bmatrix}
			\lambda_{z1} \\ \lambda_{z2}
		\end{bmatrix}
		= \mathbf{0}.
		\label{eq:new_6dof_solver_22}
	\end{align} 
	
	In comparison with Eq.~\eqref{eq:new_6dof_solver}, $(\mathbf{G}_j)_{(1:2,1:2)}$ is the first 2$\times$2 sub-matrix of $(\mathbf{F}_j)_\text{(1:2,1:3)}$. We can see that the null space vector $\mathbf{v}_2=[\lambda_{z1},\lambda_{z2},0]^{\text{T}}$ is also the null space vector of $(\mathbf{F}_j)_\text{(1:2,1:3)}$. Thus, the rank of $(\mathbf{F}_j)_\text{(1:2,1:3)}$ is 1.  	
\end{proof}

Based on Theorem~\ref{theorem:extra_constraint}, the affine transformation constraints provide extra equations for our problem. Only if $j\text{-th}$ AC is chosen to build up the world reference system $W$, two affine transformation constraints of $j\text{-th}$ AC are used in the equation system construction. Similarly, when we choose the $j'$-th AC to build up the world reference, the extra equations of the corresponding affine transformation constraints can also be provided. Thus, there are 6 extra equations for the relative pose estimation using ACs:
{\begin{gather}
		\det(\mathbf{M}(q_x, q_y, q_z)) = 0, \label{eq:submatrix_2b23extra} \\
		{\mathbf{M}\in\{2\times 2 \text{ submatrices of } {(\mathbf{F}_j)_\text{(1:2,1:3)}}\} \ \cup } \nonumber \\
		\quad \quad { \ \{2\times 2 \text{ submatrices of } {(\mathbf{F}_{j'})_\text{(1:2,1:3)}}\} \nonumber}.
\end{gather}}

These extra equations have a degree of $4$. {Note that the extra implicit constraints Eq.~\eqref{eq:submatrix_2b23extra} are independent of Eqs.~\eqref{eq:new_6dof_solver} and~\eqref{eq:new_6dof_solver2}. For geometric explanation, the extra constraints encode that the affine transformation constraints come from a perspective camera of two viewpoints. As we will see later, using the extra constraints from Theorem~\ref{theorem:extra_constraint} reduces the number of solutions.} 

\subsection{\label{sec:Solve6DOF}Polynomial System Solving}
We propose two minimal solvers based on the common configurations of two ACs in multi-camera systems, including an inter-camera solver and an intra-camera solver. The inter-camera solver uses inter-camera ACs seen by the different cameras over two views. It is suitable for multi-camera systems with large overlapping of views. The intra-camera solver uses intra-camera ACs seen by the same camera over two views. It is suitable for multi-camera systems with non-overlapping or small-overlapping of views. 

A suitable way to find algebraic solutions to the polynomial equation system Eqs.~\eqref{eq:submatrix_3by3} and~\eqref{eq:submatrix_2b23extra} is to use the Gr{\"o}bner basis technique. To keep numerical stability and avoid large number arithmetic during the calculation of Gr{\"o}bner basis, a random instance of the original equation system is constructed in a finite prime field $\mathbb{Z}_p$~\cite{lidl1997finite}. The relations between all observations are appropriately preserved. Then, we use \texttt{Macaulay~2}~\cite{grayson2002macaulay} to calculate Gr{\"o}bner basis. Finally, the solver is found with the automatic Gr{\"o}bner basis solver~\cite{larsson2017efficient}. 
We denote these polynomial equations in Eq.~\eqref{eq:submatrix_3by3} and Eq.~\eqref{eq:submatrix_2b23extra} as $\mathcal{E}_1$ and $\mathcal{E}_2$, respectively. As we will see later, the polynomial equations $\mathcal{E}_1$ and $\mathcal{E}_2$ can be extended to solve various relative pose estimation problems, such as with known rotation angle. In this paper, $\mathcal{E}_1$ is sufficient to solve the relative pose with inter-camera ACs. For intra-camera ACs, there are one-dimensional families of extraneous roots if only $\mathcal{E}_1$ is used. Moreover, using both $\mathcal{E}_1$ and $\mathcal{E}_2$ can reduce the number of solutions in the inter-camera case. 
\begin{table}[tbp]
	\centering
	\caption{Minimal solvers for 6DOF relative pose estimation of multi-camera systems. Inter-camera refers to ACs which are seen by the different cameras over two views, inter-camera refers to ACs which are seen by the same camera over two views. \texttt{\#sol} indicates the number of solutions. \texttt{$1$-dim} indicates one dimensional families of extraneous roots.}
	\begin{center}
		\setlength{\tabcolsep}{1.8mm}{
			\scalebox{1.18}{
				\begin{tabular}{c||c|c|c|c} 
					\hline
					\multirow{2}{*}{\centering AC type} & \multicolumn{2}{c|}{$\mathcal{E}_1$} &  \multicolumn{2}{c}{$\mathcal{E}_1+\mathcal{E}_2$}  \\ 
					\cline{2-5} 
					&    \#sol  & template   &   \#sol    &   template  \\ 
					\hline
					{Inter-camera} & $56$ & $56\times 120$ & $48$ &  $64\times 120$ \\ \hline
					{Intra-camera} & 1-dim & $-$ & $48$ & $72\times 120$ \\
					\hline
		\end{tabular}}}
	\end{center}
	\label{tab:complete_solution}
\end{table}

Table~\ref{tab:complete_solution} shows the resulting inter-camera and intra-camera solvers. We have the following observations. (1) If $\mathcal{E}_1$ is used, the inter-camera solver maximally has $56$ complex solutions and the elimination template of size $56\times120$. But the intra-camera case has one-dimensional families of extraneous roots. (2) If both $\mathcal{E}_1$ and $\mathcal{E}_2$ are used, the number of complex solutions obtained by the inter-camera solver can be reduced to 48. The number of complex solutions obtained by the intra-camera solver is also 48. The elimination template of the inter-camera solver and intra-camera solver is $64\times 120$ and $72\times120$, respectively. (3) Most of the solutions obtained by the proposed solvers are
usually complex solutions. When the minimal solvers are integrated within the RANSAC framework to reject outliers, these complex solutions can be removed directly in the best candidate selection stage. (4) For the inter-camera case, using equations from $\mathcal{E}_1$ results in smaller eliminate templates than using $\mathcal{E}_1+\mathcal{E}_2$. Meanwhile, the solver resulting from $\mathcal{E}_1$ has better numerical stability than the solver resulting from $\mathcal{E}_1+\mathcal{E}_2$. This phenomenon has also been observed in previous literature~\cite{byrod2009fast}, which shows that the number of basis might affect the numerical stability.   

Once the rotation parameters $\{q_x, q_y, q_z\}$ are obtained, $\mathbf{R}$ can be obtained immediately. Then $\{\lambda_{jk}\}_{k=1,2}$ and  $\{\lambda_{{j'}k}\}_{k=1,2}$ are determined by finding the null space of $\mathbf{F}_j$ and ${\mathbf{F}_{j'}}$, respectively. Note that the translations estimated by $\{\lambda_{jk}\}$ and $\{\lambda_{{j'}k}\}$ are theoretically the same in minimal problems. Take the translation estimation using $\{\lambda_{jk}\}$ for an example. We can calculate $\mathbf{t}_1$ and $\mathbf{t}_2$ by Eq.~\eqref{eq:calcu_trans}. Finally we calculate the relative pose by compositing the transformations $[\mathbf{R}_1, \mathbf{t}_1]$ and $[\mathbf{R}_2, \mathbf{t}_2]$.     

\section{\label{sec:5DOFmotion}5DOF Relative Pose Estimation for Multi-Camera Systems}
In this section, we show that the proposed minimal solver generation framework can be extended to solve 5DOF relative pose estimation with known rotation angle prior. Suppose that the relative rotation angle between view~1 and view~2 is known, which can be provided by the IMU directly, even though camera$-$IMU extrinsics are unavailable or unreliable. Three unknowns $\{q_x,q_y,q_z\}$ satisfy the constraint as follows:
\begin{equation}
	\begin{aligned}	
		q_x^2+q_y^2+q_z^2 - \tan^2(\theta/2) = 0,
	\end{aligned}		
	\label{eq:R5dof1_knownAngle}
\end{equation}
where $\theta$ is the relative rotation angle between two views. In 5DOF relative pose estimation for multi-camera systems, the inputs for the proposed solvers are two ACs and the relative rotation angle of the multi-camera system.

\subsection{Equation System Construction}
With known relative rotation angle provided by IMU, the relative pose estimation problem of multi-camera systems has 5DOF. Considering that two ACs provide six independent constraints, the number of constraints is greater than the number of unknowns, and there is a redundant constraint. Thus, we randomly choose four equations from Eq.~\eqref{eq:new_6dof_solver} to explore the minimal case solution. For example, two affine transformation constraints of $j\text{-th}$ AC, and the epipolar constraint and the first affine transformation constraint of the other AC are stacked into four equations in five unknowns, \emph{i.e.}, the first four equations of Eq.~\eqref{eq:new_6dof_solver}:
\begin{align}
	\underbrace{\mathbf{F}'_j(q_x, q_y, q_z)}_{4\times 3}
	\begin{bmatrix}
		\lambda_{j1} \\ \lambda_{j2} \\ 1
	\end{bmatrix}
	= \mathbf{0}.
	\label{eq:new_5dof_solver}
\end{align} 

Since Eq.~\eqref{eq:new_5dof_solver} has non-trival solutions, the rank of $\mathbf{F}'_j$ satisfies $\rank(\mathbf{F}'_j) \le 2$. Thus, all the $3\times 3$ sub-determinants of $\mathbf{F}'_j$ must be zero. This gives four equations about three unknowns $\{q_x, q_y, q_z\}$.

Similarly, we can choose the other AC to build up the world reference $W$. Suppose the $j'$-th AC is chosen, we build an equation system which is similar to Eq.~\eqref{eq:new_5dof_solver}:
\begin{align}
	\underbrace{\mathbf{F}'_{j'}(q_x, q_y, q_z)}_{4\times 3}
	\begin{bmatrix}
		\lambda_{{j'}1} \\ \lambda_{{j'}2} \\ 1
	\end{bmatrix}
	= \mathbf{0}.
	\label{eq:new_5dof_solver2}
\end{align}

Note that Eq.~\eqref{eq:new_5dof_solver2} provides new constraints which are different from Eq.~\eqref{eq:new_5dof_solver}. There are one dimensional families of extraneous roots if only Eq.~\eqref{eq:new_5dof_solver} or Eq.~\eqref{eq:new_5dof_solver2} is used. Based on Eqs.~\eqref{eq:new_5dof_solver} and~\eqref{eq:new_5dof_solver2}, we have eight equations with three unknowns $\{q_x, q_y, q_z\}$:
{\begin{gather}
		\det(\mathbf{N}(q_x, q_y, q_z)) = 0, \label{eq:submatrix_3by3_5dof} \\
		{\mathbf{N}\in\{3\times 3 \text{ submatrices of } {\mathbf{F}'_j}\} \cup 
			\{3\times 3 \text{ submatrices of } {\mathbf{F}'_{j'}}\} \nonumber}.
\end{gather}}

\vspace{-5pt}
\noindent These equations have a degree of $6$. Moreover, based on Proposition~\ref{theorem:extra_constraint}, the affine transformation constraints provide extra constraints for our problem. Only if $j\text{-th}$ AC is chosen to build up the world reference system $W$, two affine transformation constraints of $j\text{-th}$ AC are used in the equation system construction. Thus, we can derive extra equations from the affine transformation constraints, \emph{i.e.}, the rank of $(\mathbf{F}'_j)_\text{(1:2,1:3)}$ in Eq.~\eqref{eq:new_5dof_solver} is $1$. There are three extra equations for the 5DOF relative pose estimation problem of the multi-camera system:
{\begin{gather}
		\det(\mathbf{M}(q_x, q_y, q_z)) = 0, \label{eq:submatrix_2b23extra_5dof} \\
		{\mathbf{M}\in\{2\times 2 \text{ submatrices of } {(\mathbf{F}'_j)_\text{(1:2,1:3)}}\} \nonumber}.
\end{gather}}

\vspace{-10pt}
These three equations have a degree of $4$. Note that Eq.~\eqref{eq:submatrix_2b23extra_5dof} are independent of Eqs.~\eqref{eq:new_5dof_solver} and~\eqref{eq:new_5dof_solver2}. Using the extra constraints reduces the number of solutions.
\subsection{Polynomial System Solving}
With the known relative rotation angle of multi-camera systems, we also propose two minimal solvers based on the configurations of two ACs in the multi-camera system, including inter-camera solver and intra-camera solver. Similar to Section~\ref{sec:6DOFmotion}, the Gr{\"o}bner basis technique~\cite{larsson2017efficient} is used to produce the minimal solvers. We denote these polynomial equations in Eqs.~\eqref{eq:R5dof1_knownAngle}\eqref{eq:submatrix_3by3_5dof} and Eq.~\eqref{eq:submatrix_2b23extra_5dof} as $\mathcal{E}_1$ and $\mathcal{E}_2$, respectively. $\mathcal{E}_1$ is sufficient to solve the relative pose with both inter-camera and intra-camera ACs.  However, using both $\mathcal{E}_1$ and $\mathcal{E}_2$ can reduce the number of solutions. 

Table~\ref{tab:complete_solution_5DOF} shows the resulting inter-camera and intra-camera solvers. We have the following observations. (1) If $\mathcal{E}_1$ is used, the inter-camera solver	maximally has $42$ complex solutions and the elimination template of size $98\times164$. The intra-camera solver maximally has $44$ complex solutions and the elimination template of size $120\times164$. (2) If both $\mathcal{E}_1$ and $\mathcal{E}_2$ are used, the number of complex solutions obtained by two solvers can be reduced to 36. The elimination template of inter-camera solver and intra-camera solver are $104\times 164$ and $84\times120$, respectively. (3) The solver resulting from $\mathcal{E}_1$ has better numerical stability than the solver resulting from $\mathcal{E}_1+\mathcal{E}_2$.    
\begin{table}[tbp]
	\caption{Minimal solvers for 5DOF relative pose estimation of multi-camera systems. Inter-camera refers to ACs which are seen by the different cameras over two views, inter-camera refers to ACs which are seen by the same camera over two views. \texttt{\#sol} indicates the number of solutions. \texttt{$1$-dim} indicates one dimensional families of extraneous roots.}
	\vspace{-8pt}
	\begin{center}
		\setlength{\tabcolsep}{1.8mm}{
			\scalebox{1.18}{
				\begin{tabular}{c|c|c|c|c}
					\hline
					\multirow{2}{*}{\centering AC type} & \multicolumn{2}{c|}{$\mathcal{E}_1$} &  \multicolumn{2}{c}{$\mathcal{E}_1+\mathcal{E}_2$}  \\ 
					\cline{2-5} 
					&    \#sol  & template   &   \#sol    &   template  \\ 
					\hline
					inter-camera & $42$ & $98\times 164$ & $36$ &  $104\times 164$ \\ \hline
					intra-camera & $44$ & $120\times 164$ & $36$ & $84\times 120$ \\
					\hline
		\end{tabular} }}
	\end{center}
	\label{tab:complete_solution_5DOF}
	\vspace{-6pt}
\end{table}

Once rotation parameters $\{q_x, q_y, q_z\}$ are obtained, $\mathbf{R}$ can be obtained immediately. Then $\{\lambda_{jk}\}_{k=1,2}$ is determined by finding the null space of $\mathbf{F}_j$, see Eq.~\eqref{eq:new_5dof_solver}. Next we can calculate $\mathbf{t}_1$ and $\mathbf{t}_2$ by Eq.~\eqref{eq:calcu_trans}. Finally we calculate the relative pose of the multi-camera systems by compositing the transformations $[\mathbf{R}_1, \mathbf{t}_1]$ and $[\mathbf{R}_2, \mathbf{t}_2]$. 

\section{\label{sec:configurations}Degenerated Configurations}
In this section, we prove three cases of critical motions for the proposed solvers in Sections~\ref{sec:6DOFmotion} and~\ref{sec:5DOFmotion}, including both the inter-camera solvers and the intra-camera solvers. In these critical configurations, the rotation and the translation direction between two views of the multi-camera system can be correctly recovered, but the metric scale of translation is unobtainable. The proofs of degenerated configurations modeled in different ways have been proposed in~\cite{Guan_ICRA2021}.

\begin{proposition}
	\label{theorem:inter_cam}
	For inter-camera ACs, if a multi-camera system undergoes pure translation and the baseline of two cameras is parallel with the translation direction, the metric scale of translation cannot be recovered.
\end{proposition}
\begin{proof}
	In the case of inter-camera ACs, each AC is seen by the different cameras over two consecutive views. For the pure translation case, the rotation between two views of the multi-camera system satisfies ${\mathbf{R}=\mathbf{I}}$. Since the baseline of two cameras is parallel with the translation direction, the translation satisfies ${\mathbf{s}}_{2}-{\mathbf{s}}_{1}={a}({\mathbf{t}_2 - \mathbf{t}_1})$, where $a$ is an unknown number, ${\mathbf{t}_2 - \mathbf{t}_1}$ is the translation between two views of the multi-camera system. The essential matrix in Eq.~\eqref{eq:essential_matrix} can be written as
	\begin{align}
		\mathbf{E}' &= \Q_2^T \left( [\mathbf{t}_2 - \mathbf{t}_1]_\times - [\mathbf{s}_2 - \mathbf{s}_1]_\times  \right) \Q_1 \nonumber \\
		&= \left(1-a \right)\Q_2^T {[\mathbf{t}_2 - \mathbf{t}_1]_\times} \Q_1.
		\label{eq:GECS5dof_DegenCase1}
	\end{align}
	
	The essential matrix $\mathbf{E}'$ is homogeneous with the translation between two views of the multi-camera system ${\mathbf{t}_2 - \mathbf{t}_1}$. We substitute Eq.~\eqref{eq:GECS5dof_DegenCase1} into Eqs.~\eqref{eq:constraint_epipolar_single} and~\eqref{eq:constraint_affine_single}. Then the geometric constraints provided by an AC become:
	\begin{align}
		&\x_j'^T {\Q_2^T {[\mathbf{t}_2 - \mathbf{t}_1]_\times} \Q_1} \x_j = 0, \label{eq:AC_Constraints_Degen1} \\
		(\Q_1^T {[\mathbf{t}_2 - \mathbf{t}_1]_\times} &\Q_2 \x_j')_{(1:2)} = \A_j^{T} (\Q_2^T {[\mathbf{t}_2 - \mathbf{t}_1]_\times} \Q_1 \x_j)_{(1:2)}.
		\label{eq:AC_Constraints_Degen2}
	\end{align}
	
	Suppose $\kappa$ is a free parameter, it can be verified that $\kappa({\mathbf{t}_2 - \mathbf{t}_1})$ satisfies Eqs.~\eqref{eq:AC_Constraints_Degen1} and~\eqref{eq:AC_Constraints_Degen2}. Thus, the metric scale of translation between two views of the multi-camera system cannot be recovered.	
\end{proof}

\begin{figure}[tbp]
	\begin{center}
		\includegraphics[width=0.7\linewidth]{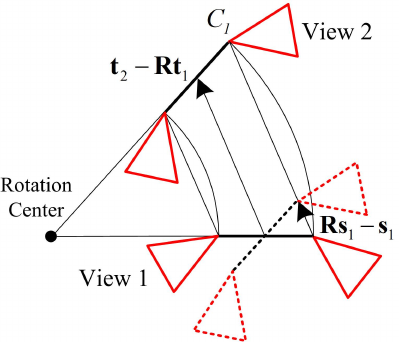}
	\end{center}
	\caption{Critical motion due to constant rotation rate.}
	\label{fig:Criticalmotion}
\end{figure}

\begin{proposition}
	\label{theorem:intra_cam}
	For intra-camera ACs, when a multi-camera system undergoes pure translation or constant rotation rate, both cases are degenerate motions. Specifically, the metric scale of translation cannot be recovered.
\end{proposition}
\begin{proof}
	In the case of intra-camera ACs, each AC is seen by the same camera over two consecutive views. So we have ${\mathbf{s}}_{1}={\mathbf{s}}_{2}$ and ${\mathbf{Q}_{1}}={\mathbf{Q}_{2}}$. 
	
	(1) For the pure translation case, with the assumption that ${\mathbf{R}=\mathbf{I}}$, the essential matrix in Eq.~\eqref{eq:essential_matrix} can be written as 
	\begin{align}
		\mathbf{E}' &= \Q_1^T \left( [\mathbf{t}_2 - \mathbf{t}_1]_\times\right) \Q_1.
		\label{eq:GECS5dof_DegenCase2}
	\end{align}
	
	The essential matrix is homogeneous with the translation between two views of the multi-camera system ${\mathbf{t}_2 - \mathbf{t}_1}$. Suppose $\kappa$ is a free parameter, it can be verified that $\kappa({\mathbf{t}_2 - \mathbf{t}_1})$ invariably satisfies Eqs.~\eqref{eq:constraint_epipolar_single} and~\eqref{eq:constraint_affine_single}.
	
	(2) For the constant rotation rate case,~\emph{i.e.}, both camera paths move along concentric circles, the proof is inspired by~\cite{clipp2008robust}. We take the camera $C_1$ in Fig.~\ref{fig:Criticalmotion} as an example, the rotation induced translation ${{\mathbf{R}}{\mathbf{s}_{1}}-{\mathbf{s}_{1}}}$ is aligned with the translation ${\mathbf{t}_2 - {\mathbf{R}}\mathbf{t}_1}$. Denote ${{\mathbf{R}}{\mathbf{s}_{1}}-{\mathbf{s}_{1}}} = a({\mathbf{t}_2 - {\mathbf{R}}\mathbf{t}_1})$ and substitute it to Eq.~\eqref{eq:essential_matrix}, the essential matrix becomes
	\begin{align}
		\mathbf{E}' &= \left(1+a \right)\Q_1^T \left( [\mathbf{t}_2 - {\mathbf{R}}\mathbf{t}_1]_\times {\mathbf{R}} \right) \Q_1.
		\label{eq:GECS5dof_DegenCase3}
	\end{align}
	
	The essential matrix is homogeneous with the translation between two views of the multi-camera system ${\mathbf{t}_2 - {\mathbf{R}}\mathbf{t}_1}$. Suppose $\kappa$ is a free parameter, it can be verified that $\kappa({\mathbf{t}_2 - {\mathbf{R}}\mathbf{t}_1})$  also satisfies Eqs.~\eqref{eq:constraint_epipolar_single} and~\eqref{eq:constraint_affine_single}.
\end{proof}	

To deal with these degenerate cases, we can use auxiliary sensors, such as integrating the acceleration over time from an IMU, to recover the metric scale of the translation~\cite{liu2017robust,Guan_ICRA2021,guanICCV2021minimal}. Moreover, in the absence of auxiliary sensors, since the frame rate of current cameras is high and the multi-camera system usually moves at a constant speed within a short time, we can also use the metric scale of the previous image pairs to approximate the current metric scale. 

\section{\label{sec:experiments}Experiments}
We validate the performance of our solvers using both virtual and real multi-camera systems. For 6DOF relative pose estimation, the solvers proposed in Section~\ref{sec:6DOFmotion} are referred to as \texttt{2AC method}, specifically \texttt{2AC-inter} for inter-camera ACs and \texttt{2AC-intra} for intra-camera ACs. To further distinguish two solvers for inter-camera ACs, \texttt{2AC-inter-56} and \texttt{2AC-inter-48} are used to refer to the minimal solvers resulting from $\mathcal{E}_1$ and $\mathcal{E}_1+\mathcal{E}_2$, respectively. 

For 5DOF relative pose estimation with known relative rotation angle, the solvers proposed in Section~\ref{sec:5DOFmotion} are referred to as \texttt{2AC-ka method}, specifically \texttt{2AC-ka-inter} for inter-camera ACs and \texttt{2AC-ka-intra} for intra-camera ACs. To further distinguish two solvers for inter-camera ACs, \texttt{2AC-inter-42} and \texttt{2AC-inter-36} are used to refer to the minimal solvers resulting from $\mathcal{E}_1$ and $\mathcal{E}_1+\mathcal{E}_2$, respectively. Meanwhile, to further distinguish two solvers for intra-camera ACs, \texttt{2AC-intra-44} and \texttt{2AC-intra-36} are used to refer to the solvers resulting from $\mathcal{E}_1$ and $\mathcal{E}_1+\mathcal{E}_2$, respectively. 

The proposed solvers are implemented in C++. The \texttt{2AC method} and the \texttt{2AC-ka method} are compared with state-of-the-art methods including \texttt{17PC-Li}~\cite{li2008linear}, \texttt{8PC-Kneip}~\cite{kneip2014efficient}, \texttt{6PC-Stew{\'e}nius}~\cite{henrikstewenius2005solutions}, \texttt{6AC-Ventura}~\cite{alyousefi2020multi} and \texttt{5PC-Martyushev}~\cite{martyushev2020efficient}. Note that the \texttt{5PC-Martyushev} solver also uses the known relative rotation angle as a prior. The proposed solvers are not compared with the methods which estimate the relative pose with different motion prior, such as planar motion~\cite{guanICCV2021minimal}, known gravity direction~\cite{sweeney2014solving,guanICCV2021minimal} and first-order approximation assumption~\cite{ventura2015efficient,Guan_ICRA2021}. All the solvers are integrated into RANSAC in order to remove outlier matches of the feature correspondences. The relative pose which produces the most inliers is used to measure the relative pose error. This also allows us to select the best candidate from multiple solutions by counting their inliers.

The relative rotation and translation of the multi-camera systems are compared separately in the experiments. The rotation error compares the angular difference between the ground truth rotation and the estimated rotation: ${\varepsilon _{\bf{R}}} = \arccos ((\trace({\mathbf{R}_{gt}}{{\mathbf{R}^{\text{T}}}}) - 1)/2)$, where $\mathbf{R}_{gt}$ and ${\mathbf{R}}$ denote the ground truth rotation and the corresponding estimated rotation, respectively. We evaluate the translation error by following the definition in~\cite{hee2014relative}: ${\varepsilon _{\bf{t}}} = 2\left\| ({{\mathbf{t}_{gt}}}-{\mathbf{t}})\right\|/(\left\| {\mathbf{t}_{gt}} \right\| + \left\| {{\mathbf{t}}} \right\|)$, where $\mathbf{t}_{gt}$ and ${\mathbf{t}}$ denote the ground truth translation and the corresponding estimated translation, respectively. ${\varepsilon _{\bf{t}}}$ denotes both the metric scale error and the direction error of the translation. The translation direction error is also evaluated separately by comparing the angular difference between the ground truth translation and the estimated translation: $\varepsilon_{\mathbf{t},\text{dir}} = \arccos (({{\mathbf{t}_{gt}^{\text{T}}}}{\mathbf{t}})/(\left\| {\mathbf{t}_{gt}} \right\| \cdot \left\| {{\mathbf{t}}} \right\|))$. 

Due to space limitations, the experiment results of the 5DOF relative pose estimation solvers on synthetic data and real data are provided in the supplementary material.
\begin{table*}[tbp]
	\caption{Runtime comparison of relative pose estimation solvers for multi-camera systems (unit:~$\mu s$).}
	\begin{center}
		\setlength{\tabcolsep}{1.0mm}{
			\scalebox{0.65}{
				\begin{tabular}{c||c|c|c|c|c|c|c|c|c|c|c|c}
					\hline						
					\small{Methods} &  \small{17PC-Li~\cite{li2008linear}} & \small{8PC-Kneip~\cite{kneip2014efficient}} &  \small{6PC-Stew.~\cite{henrikstewenius2005solutions}} &\small{6AC-Vent.~\cite{alyousefi2020multi}} & \small{5PC-Mart.~\cite{martyushev2020efficient}} & \small{\textbf{2AC-inter-56}}& \small{\textbf{2AC-inter-48}}& \small{\textbf{2AC-intra}} & \small{\textbf{2AC-ka-inter-42}} & \small{\textbf{2AC-ka-inter-36}} & \small{\textbf{2AC-ka-intra-44}} & \small{\textbf{2AC-ka-intra-36}}\\
					\hline
					\small{Runtime}& 43.3 & 102.0& 3275.4& 38.1 & 557.3& 1084.8 & 842.3 & 871.6& 1017.3& 930.8& 1236.8& 672.6\\
					\hline
		\end{tabular}}}
	\end{center}
	\label{tab:SolverTime_generalized}
\end{table*}

\begin{figure*}[tbp]
	\begin{center}
		\includegraphics[width=0.55\linewidth]{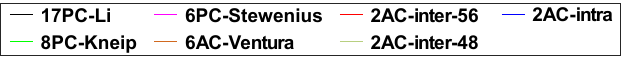}\\
		\subfigure[$\varepsilon_{\R}$]
		{
			\includegraphics[width=0.3\linewidth]{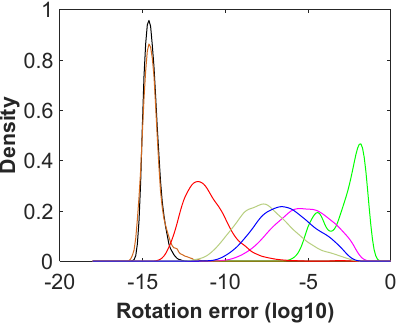}
		} 
		\subfigure[$\varepsilon_{\mathbf{t}}$]
		{
			\includegraphics[width=0.3\linewidth]{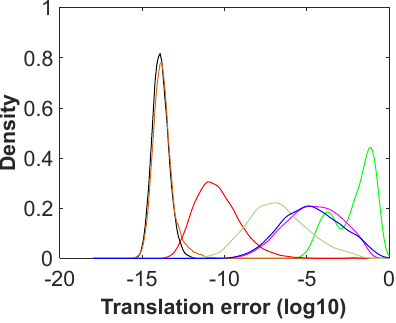}
		}
		\subfigure[$\varepsilon_{\mathbf{t},\text{dir}}$]   
		{
			\includegraphics[width=0.3\linewidth]{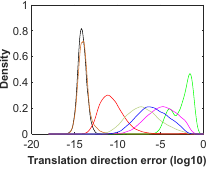}
		}
	\end{center}
	\caption{Probability density functions over 6DOF relative pose estimation errors on noise-free observations for multi-camera systems. The horizontal axis represents the $\log_{10}$ errors, and the vertical axis represents the density.}
	\label{fig:Numerical}
\end{figure*} 

\subsection{Efficiency Comparison and Numerical Stability}
The proposed solvers are evaluated on an Intel(R) Core(TM) i7-7800X 3.50GHz. All comparison solvers are implemented in C++. The \texttt{17PC-Li}, \texttt{8PC-Kneip} and \texttt{6PC-Stew{\'e}nius} are provided by OpenGV library~\cite{kneip2014opengv}. The \texttt{6AC-Ventura} is publicly available from the code of~\cite{alyousefi2020multi}. We converted the original Matlab code of \texttt{5PC-Martyushev} to C++ version~\cite{martyushev2020efficient}. Table~\ref{tab:SolverTime_generalized} shows the average processing times of the solvers over 10,000 runs. The methods \texttt{17PC-Li} and \texttt{6AC-Ventura} have low runtime, because they solve for the multi-camera motion linearly. Among the minimal solvers of 6DOF relative pose estimation, all the proposed solvers \texttt{2AC-inter-56}, \texttt{2AC-inter-48} and \texttt{2AC-intra} are significantly more efficient than the \texttt{6PC-Stew{\'e}nius} solver. It can also be seen that using the extra constraints about the affine transformation constraints makes the proposed AC-based solvers more efficient.
 	
Figure~\ref{fig:Numerical} reports the numerical stability comparison of the solvers in noise-free cases. We repeat the procedure 10,000 times and plot the empirical probability density functions as the function of the $\log_{10}$ estimated errors. Numerical stability represents the round-off error of solvers in noise-free cases. The solvers \texttt{17PC-Li} and \texttt{6AC-Ventura} achieve the best numerical stability, because the linear solvers with smaller computation burden have less round-off error. Since the \texttt{8PC-Kneip} solver adopts the iterative optimization, it is easy to fall into local minima. Among the minimal solvers, the numerical stability of all the proposed solvers \texttt{2AC-inter-56}, \texttt{2AC-inter-48} and \texttt{2AC-intra} is better than the \texttt{6PC-Stew{\'e}nius} solver. Moreover, the \texttt{2AC-inter-56} solver has better numerical stability than the \texttt{2AC-inter-48} solver, which shows that adding the extra equations $\mathcal{E}_2$ is not helpful in improving the numerical stability of the \texttt{2AC-inter} solver. Even though \texttt{2AC-inter-48} produces fewer  solutions, we prefer to perform \texttt{2AC-inter-56} for the sake of numerical accuracy in the follow-up experiments. 

In addition to efficiency and numerical stability, another important factor for a solver is the minimal number of needed feature correspondences between two views. Since the proposed solvers require only two ACs, the number of RANSAC iterations is obviously lower than PC-based methods. Thus, our solvers have an advantage in detecting the outlier and estimating the initial motion efficiently when integrating them into the RANSAC framework. See supplementary material for details. As we will see later, our minimal solvers have better overall efficiency than the comparison solvers in real-world data experiments. 

\subsection{Experiments on Synthetic Data}
The proposed inter-camera and intra-camera solvers are evaluated simultaneously with a simulated multi-camera system~\cite{guanICCV2021minimal, guan2023minimal}. The baseline length between two simulated cameras is set to 1 meter, and the movement length of the multi-camera system is set to 3 meters. The resolution of cameras is 640 $\times$ 480 pixels with a focal length of 400 pixels. The principal points are set to the image center (320, 240). We carry out a total of 1000 trials. The rotation and translation errors are assessed by the median of errors in the synthetic experiment.
\begin{figure*}[tbp]
	\begin{center}
		\includegraphics[width=0.7\linewidth]{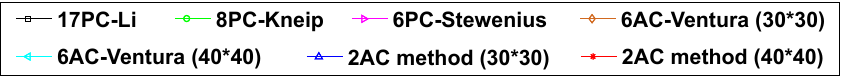}\\
		\subfigure[${\varepsilon_{\bf{R}}}$ with image noise]
		{
			\includegraphics[width=0.305\linewidth]{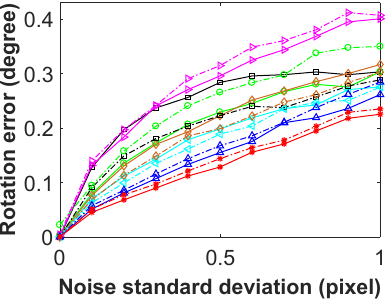}
		}
		\subfigure[${\varepsilon_{\bf{t}}}$ with image noise]
		{
			\includegraphics[width=0.305\linewidth]{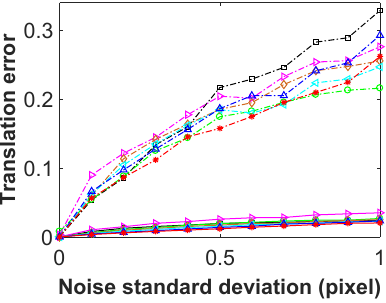}
		}
		\subfigure[$\varepsilon_{\mathbf{t},\text{dir}}$ with image noise]
		{
			\includegraphics[width=0.305\linewidth]{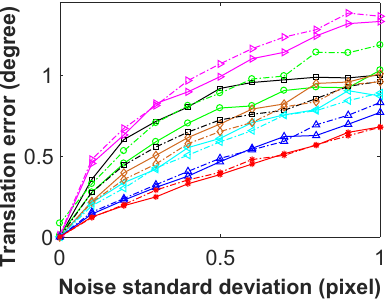}
		}
		\subfigure[${\varepsilon_{\bf{R}}}$ with image noise]
		{
			\includegraphics[width=0.305\linewidth]{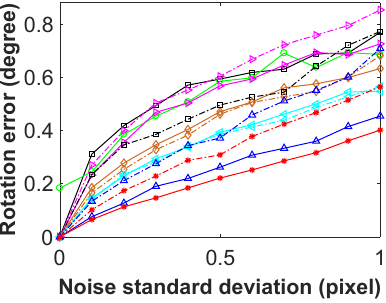}
		}
		\subfigure[${\varepsilon_{\bf{t}}}$ with image noise]
		{
			\includegraphics[width=0.305\linewidth]{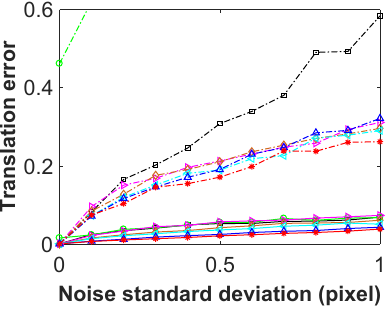}
		}
		\subfigure[$\varepsilon_{\mathbf{t},\text{dir}}$ with image noise]
		{
			\includegraphics[width=0.305\linewidth]{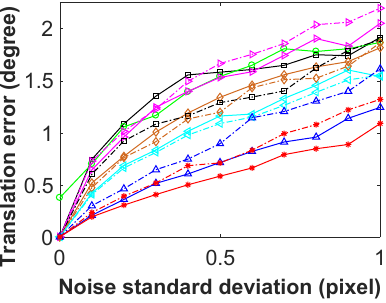}
		}
		\subfigure[${\varepsilon_{\bf{R}}}$ with image noise]
		{
			\includegraphics[width=0.305\linewidth]{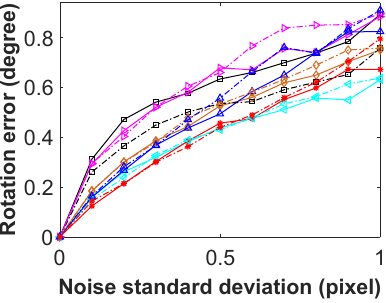}
		}
		\subfigure[${\varepsilon_{\bf{t}}}$ with image noise]
		{
			\includegraphics[width=0.305\linewidth]{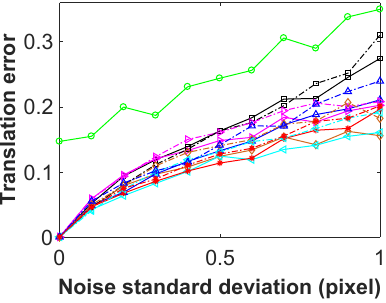}
		}
		\subfigure[$\varepsilon_{\mathbf{t},\text{dir}}$ with image noise]
		{
			\includegraphics[width=0.305\linewidth]{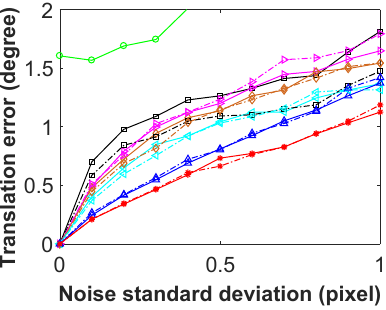}
		}
	\end{center}
	\caption{Rotation and translation error with increasing image noise. The first, second and third rows report the performance of the proposed 6DOF relative pose estimation solvers under forward, random, and sideways motions, respectively. Solid line indicates that inter-camera ACs are used, and dash-dotted line indicates that intra-camera ACs are used.}
	\label{fig:RT_generalized_Pix}
\end{figure*}

In each test, 100 ACs are generated randomly, including 50 ACs from a ground plane and 50 ACs from 50 random planes. The synthetic scene is randomly produced in a cubic region of size $[-5,5]$$\times$$[-5,5]$$\times$$[10,20]$ meters. To obtain each AC, the PC is computed by reprojecting a random 3D point from a plane into two cameras. The associated affine transformation is computed as follows: First, four additional image points are chosen as the vertices of a square in view~1, where its center is the PC of AC. The side length of the square is set to $30$ or $40$~pixels. A larger side length means the larger support regions for generating the ACs, which causes smaller noise of affine transformation. The support region is used for AC noise simulation only. Second, the ground truth homography is used to calculate the four corresponding image points in view~2. Third, Gaussian noise is added to the coordinates of four sampled image point pairs. Fourth, the noisy affine transformation is calculated from the first-order approximation of the noisy homography, which is estimated by using four noise image point pairs. This approach encompasses an indirect yet geometrically interpretable process of noising the affine transformation~\cite{barath2019homography}. Gaussian noise with a defined standard deviation is added to both the PCs and the sampled image point pairs for estimating the affine transformations. In the experiments, the required number of ACs is randomly chosen for the solvers within the RANSAC framework. For the PC-based solvers, only the PCs derived from the ACs are used.
\begin{figure}[t]
	\begin{center}
		\includegraphics[width=0.95\linewidth]{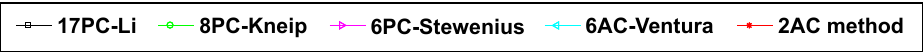}\\
		\subfigure[\scriptsize{${\varepsilon _{\bf{R}}}$ with image noise}]
		{
			\includegraphics[width=0.95\linewidth]{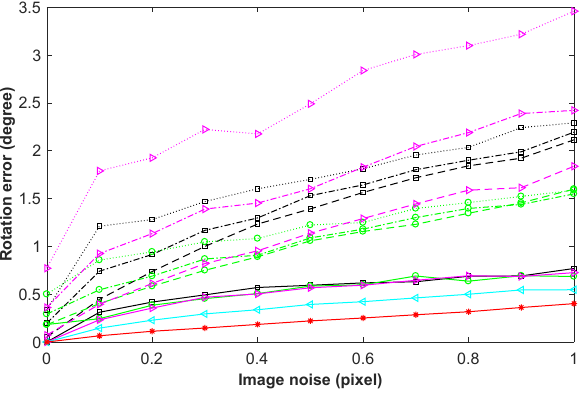}
		}\\
		\subfigure[\scriptsize{${\varepsilon _{\bf{t}}}$ with image noise}]
		{
			\includegraphics[width=0.95\linewidth]{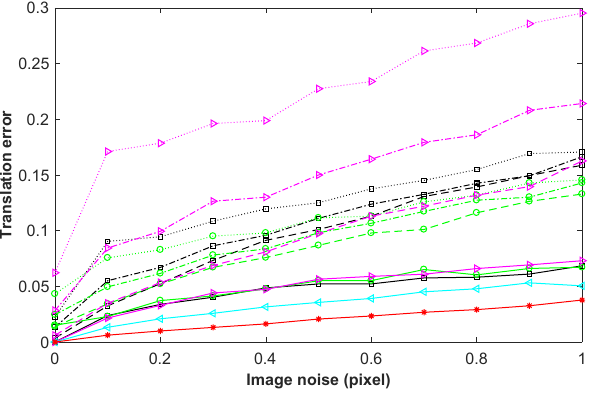}
		}\\
		\subfigure[\scriptsize{$\varepsilon_{\mathbf{t},\text{dir}}$ with image noise}]
		{
			\includegraphics[width=0.95\linewidth]{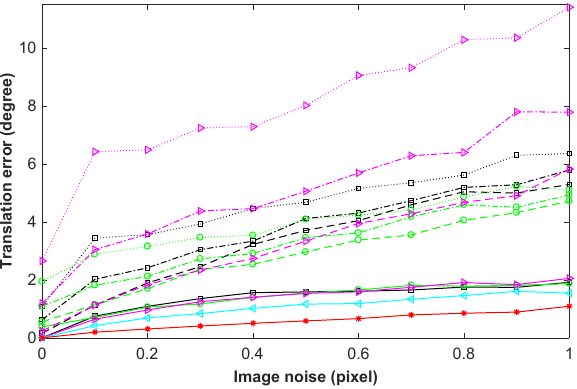}
		}
	\end{center}
	\vspace{-10pt} 
	\caption{Rotation and translation error of the PC-based solvers using ACs with increasing image noise. Solid lines indicate using the image point pairs of ACs. Dashed lines, dash-dotted lines, and dotted lines indicate using the hallucinated PCs converted from ACs, when the size of the distribution area is set to 1, 5, and 10 pixels, respectively.}		
	\vspace{-10pt} 
	\label{fig:RTRandomMotion_1ACto3PCs_stew}
\end{figure}

\subsubsection{\label{sec:imagenoise}Accuracy with Image Noise}
The performance of the proposed solvers is evaluated under varying levels of image noise. In this context, the image noise magnitude is set to Gaussian noise with a standard deviation, which ranges from 0 to 1.0 pixels. The motion directions of multi-camera systems are configured as forward, random, and sideways, respectively. Figure~\ref{fig:RT_generalized_Pix} illustrates the performance of the proposed solvers against increasing levels of image noise. Comprehensive evaluations of all the solvers are carried out on both inter-camera and intra-camera ACs. The corresponding estimates are depicted using solid and dash-dotted lines, respectively. The \texttt{2AC method} serves to signify \texttt{2AC-inter-56} when implementing inter-camera ACs, and \texttt{2AC-intra} when implementing intra-camera ACs. To ensure the presentation effect, the display range of the figures is limited, resulting in some curves with large errors being invisible or partially invisible.

We have made the following observations: (1) The solvers using inter-camera ACs generally exhibit superior performance compared to those using intra-camera ACs, particularly in the recovery of metric scale of translation. (2) The performance of AC-based solvers is influenced by the magnitude of affine transformation noise, which is determined by the support region of sampled points. Therefore, AC-based solvers exhibit better performance with larger support regions at the same magnitude of image noise. (3) At a square side length of $40$~pixels, the proposed \texttt{2AC method} yields better results in comparison to other methods using both inter-camera and intra-camera ACs. (4) The \texttt{8PC-Kneip} works well in the forward motion of multi-camera systems, but fares poorly in random and sideways motions. This may be attributed to iterative optimization that is susceptible to falling into local minima~\cite{ventura2015efficient}. (5) The linear solvers \texttt{17PC-Li} and \texttt{6AC-Ventura} with fewer calculations have less round-off error than the proposed \texttt{2AC method} in noise-free cases, see Fig.~\ref{fig:Numerical}. However, our method has better accuracy than the linear solvers with the influence of image noise. This is also consistent with the real-world data experiments.

\subsubsection{\label{sec:generatedPCs_sm}Evaluation of PC-based Solvers using ACs}
In this experiment, we test the performance of PC-based solvers for the multi-camera relative pose estimation using ACs. An AC can be converted into three PCs, which are then used as the input of the  PC-based solvers. Three generated PCs converted from an AC consist of a PC $({\mathbf{x}}_j, {\mathbf{x}}'_j)$ and two hallucinated PCs calculated by the local affine transformation $\A_j$. However, the hallucinated PCs inevitably have errors even for noise-free input. Because the local affine transformation is only valid in the distribution area, where it is infinitesimally close to the image coordinates of AC~\cite{barath2018efficient}. Following the conversion equation in~\cite{barath2020making}, we can compute three approximate PCs converted from one AC: $({\mathbf{x}}_{j}, {\mathbf{x}}_{j} + [s, \ 0]^T, {\mathbf{x}}_{j} + [0, \ s]^T)$ and $({\mathbf{x}}'_{j}, {\mathbf{x}}'_{j} + \mathbf{A}_{j}[s, \ 0]^T, {\mathbf{x}}'_{j} + \mathbf{A}_{j}[0, \ s]^T)$, where $s$ represents the size of the distribution area of the generated PCs. It can be found that the size of $s$ determines the magnitude of the conversion error of the hallucinated PCs. In this experiment, we set $s$ to 1, 5, and 10 pixels, respectively. The performance of PC-based solvers is evaluated with the different sizes of the distribution area.

Take the relative pose estimation using inter-camera ACs for an example. We carry out a total of 1000 trials in the synthetic experiment. In each test, 100 ACs are generated randomly, and the support region for generating the ACs is set to 40*40 pixels. In this experiment, the \texttt{2AC method} indicates the proposed \texttt{2AC-inter-56} solver. The required ACs are selected randomly for the AC-based solvers within the RANSAC scheme. So, 6 and 2 ACs are selected randomly for the \texttt{6AC-Ventura}~\cite{alyousefi2020multi} method and the proposed \texttt{2AC method}, respectively. For the PC-based solvers, the hallucinated PCs converted from a minimal number of ACs are used as input. Thus, 6, 3, and 2 ACs are selected randomly for the solvers \texttt{17PC-Li}~\cite{li2008linear}, \texttt{8PC-Kneip}~\cite{kneip2014efficient}, and  \texttt{6PC-Stew{\'e}nius}~\cite{henrikstewenius2005solutions}, respectively. It should be noted that we only use the hallucinated PCs converted from ACs for hypothesis generation. The corresponding inlier set of the estimated relative pose is still determined by evaluating the image point pairs of ACs. The relative pose which produces the most inliers is used to measure the error. This also allows us to select the best candidate from multiple solutions.

Figure~\ref{fig:RTRandomMotion_1ACto3PCs_stew} shows the performance of the PC-based solvers with increasing image noise under random motion. Solid lines represent the estimation results using the image point pairs of ACs. Dashed lines, dash-dotted lines, and dotted lines represent the estimation results using the hallucinated PCs converted from ACs, when the size of the distribution area is set to 1, 5, and 10 pixels, respectively. We have the following observations. (1) The PC-based solvers using the hallucinated PCs have worse performance than those using the image point pairs of the ACs. Because the conversion error is newly introduced while the hallucinated PCs are generated by the ACs. In addition, since the hallucinated PCs generated by each AC are close to each other, this may be a degenerate case for the PC-based solvers. (2) Even though the image noise is zero, the rotation and translation error of the PC-based solvers is not zero when using the hallucinated PCs. This also shows that the local affine transformation is only valid in the infinitesimal patches around the image point pairs of ACs. (3) The PC-based solvers have better performance with smaller distribution areas at the same magnitude of image noise. Because the conversion error between ACs and hallucinated PCs is determined by the size of the distribution area, and the smaller distribution area causes a smaller conversion error. (4) The proposed \texttt{2AC method} provides better estimation results than the comparative methods. Compared with the PC-based solvers, the AC-based solvers use the affine transformation constraints as expressed in Eq.~\eqref{eq:constraint_affine_single} in the paper. These affine transformation constraints describe the strictly satisfied geometric relationship between the essential matrix and the local affine transformation. The affine transformation constraints do not have any conversion error. It is an advantage compared to using the epipolar constraints of the hallucinated PCs.

\subsection{Experiments on Real Data}
The performance of the proposed solvers is evaluated on three public datasets in popular modern robot applications. Specifically, the \texttt{KITTI} dataset~\cite{geiger2013vision} and \texttt{nuScenes} dataset~\cite{Caesar_2020_CVPR} are collected on an autonomous driving environment. The \texttt{EuRoc} MAV dataset~\cite{burri2016euroc} is collected on an unmanned aerial vehicle environment. These datasets provide challenging image pairs, such as large motion and highly dynamic scenes. We compare the performance of the proposed solvers against state-of-the-art 6DOF relative pose estimation techniques. The accuracy of all the solvers is evaluated using the rotation error ${\varepsilon_{\bf{R}}}$ and the translation direction error $\varepsilon_{\mathbf{t},\text{dir}}$~\cite{alyousefi2020multi,kneip2014efficient,liu2017robust}. We tested on 30,000 image pairs in total. Our solvers focus on relative pose estimation, \emph{i.e.}, integrating the minimal solver with RANSAC. To ensure the fairness of the experiments, the PCs derived from the ACs are used in the PC-based solvers.
\begin{table}[tbp]
	\caption{Rotation and translation error for \texttt{KITTI} sequences (unit: degree).}
	\vspace{-5pt}
	\begin{center}
		\setlength{\tabcolsep}{1.0mm}{
			\scalebox{0.74}{
				\begin{tabular}{c||c|c|c|c|c}
					\hline
					\multirow{2}{*}{\footnotesize{Seq.}} &  
					\footnotesize{17PC-Li}\scriptsize{~\cite{li2008linear}} &  \footnotesize{8PC-Kneip}\scriptsize{~\cite{kneip2014efficient}} &  \footnotesize{6PC-Stew.}\scriptsize{~\cite{henrikstewenius2005solutions}}& \footnotesize{6AC-Vent.}\scriptsize{~\cite{alyousefi2020multi}}& \footnotesize{\textbf{2AC method}} \\
					\cline{2-6}
					& ${\varepsilon _{\bf{R}}}$\quad\ $\varepsilon_{\mathbf{t},\text{dir}}$      &  ${\varepsilon _{\bf{R}}}$\quad\ $\varepsilon_{\mathbf{t},\text{dir}}$      &   ${\varepsilon _{\bf{R}}}$\quad\ $\varepsilon_{\mathbf{t},\text{dir}}$     &   ${\varepsilon _{\bf{R}}}$\quad\ $\varepsilon_{\mathbf{t},\text{dir}}$  &   ${\varepsilon _{\bf{R}}}$\quad\ $\varepsilon_{\mathbf{t},\text{dir}}$\\
					\hline
					\rowcolor{gray!10}\small{00 (4541 images)}&       0.139 \ 2.412 &  0.130  \  2.400& 0.229 \ 4.007 & 0.142 \ 2.499 &\textbf{0.121} \ \textbf{2.184}     \\
					\small{01 (1101 images)}&                         0.158 \ 5.231 &  0.171  \  4.102& 0.762 \ 41.19 & 0.146 \ 3.654 &\textbf{0.136} \ \textbf{2.821}     \\
					\rowcolor{gray!10}\small{02 (4661 images)}&       0.123 \ 1.740 &  0.126  \  1.739& 0.186 \ 2.508 & 0.121 \ 1.702 &\textbf{0.120} \ \textbf{1.696}     \\
					\small{03 \ \ (801 images)} &                     0.115 \ 2.744 &  0.108  \  2.805& 0.265 \ 6.191 & 0.113 \ 2.731 &\textbf{0.097} \ \textbf{2.428}     \\
					\rowcolor{gray!10}\small{04 \ \ (271 images)} &   0.099 \ 1.560 &  0.116  \  1.746& 0.202 \ 3.619 & 0.100 \ 1.725 &\textbf{0.090} \ \textbf{1.552}     \\
					\small{05 (2761 images)}&                         0.119 \ 2.289 &  0.112  \  2.281& 0.199 \ 4.155 & 0.116 \ 2.273 &\textbf{0.103} \ \textbf{2.239}     \\
					\rowcolor{gray!10}\small{06 (1101 images)}&       0.116 \ 2.071 &  0.118  \  1.862& 0.168 \ 2.739 & 0.115 \ 1.956 &\textbf{0.106} \ \textbf{1.788}     \\
					\small{07 (1101 images)}&                         0.119 \ 3.002 &\textbf{0.112}  \  3.029& 0.245 \ 6.397 & 0.137 \ 2.892 &  0.123 \ \textbf{2.743}     \\
					\rowcolor{gray!10}\small{08 (4071 images)}&       0.116 \ 2.386 &  0.111  \  2.349& 0.196 \ 3.909 & 0.108 \ 2.344 &\textbf{0.089} \ \textbf{2.235}     \\
					\small{09 (1591 images)}&                         0.133 \ 1.977 &  0.125  \  1.806& 0.179 \ 2.592 & 0.124 \ 1.876 &\textbf{0.116} \ \textbf{1.644}     \\
					\rowcolor{gray!10}\small{10 (1201 images)}&       0.127 \ 1.889 &\textbf{ 0.115}  \  1.893& 0.201 \ 2.781 & 0.203 \ 2.057 &0.184  \ \textbf{1.687}     \\
					\hline					
		\end{tabular}}}
	\end{center}
	\label{tab:RTErrror_kitti_generalized}
\end{table}

\begin{table}[tbp]
	\caption{Runtime of RANSAC averaged over \texttt{KITTI} sequences (unit:~$s$).}
	\vspace{-5pt}
	\begin{center}
		\setlength{\tabcolsep}{1.0mm}{
			\scalebox{0.67}{
				\begin{tabular}{c||c|c|c|c|c}
					\hline
					\small{Methods} &  \small{17PC-Li}\scriptsize{~\cite{li2008linear}} &  \small{8PC-Kneip}\scriptsize{~\cite{kneip2014efficient}} &  \small{6PC-Stew.}\scriptsize{~\cite{henrikstewenius2005solutions}}& \small{6AC-Vent.}\scriptsize{~\cite{alyousefi2020multi}}& \small{\textbf{2AC method}} \\
					\hline
					\small{Mean time }& 52.82 & 10.36 & 79.76& 6.83& 4.87\\
					\hline
					\small{Standard deviation}& 2.62 & 1.59 & 4.52& 0.61& 0.35\\
					\hline
		\end{tabular}}}
	\end{center}
	\label{RANSACTime_generalized}
\end{table}

\begin{figure*}[t]
	\begin{center}
		\subfigure[\texttt{8PC-Kneip}]
		{
			\includegraphics[width=0.38\linewidth]{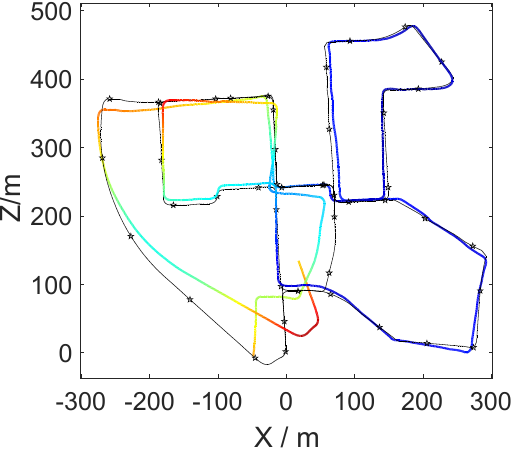}
		}
		\hspace{0.18in}
		\subfigure[\texttt{2AC method}]
		{
			\includegraphics[width=0.462\linewidth]{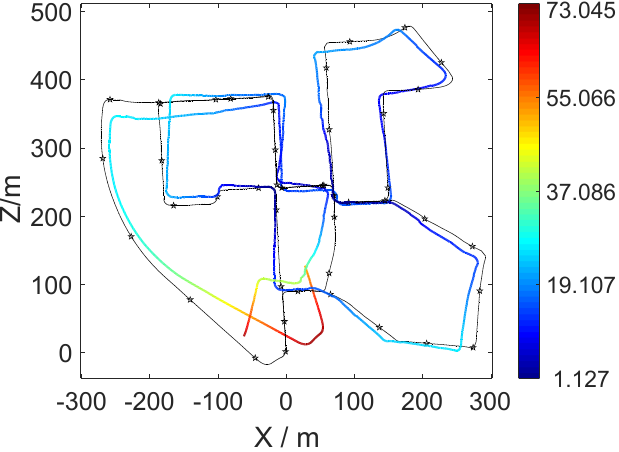}
		}
	\end{center}
	\caption{Estimated trajectories without any post-refinement. The relative pose measurements between consecutive frames are directly concatenated. The trajectories estimated by \texttt{8PC-Kneip}~\cite{kneip2014efficient} and \texttt{2AC method} are represented by the colorful curves. The ground truth trajectory is represented by the black curves with stars. Best viewed in color.}
	\label{fig:trajectory}
\end{figure*}

\subsubsection{Experiments on KITTI Dataset}
All the solvers are evaluated on \texttt{KITTI} dataset~\cite{geiger2013vision}, which is collected outdoors using forward-facing stereo cameras mounted on autonomous vehicles. We consider it as a general multi-camera system by disregarding the overlap in their fields of view. The \texttt{2AC-intra} solver is tested on all the available 11 sequences, which consist of 23000 image pairs in total. The ground truth is directly given by the output of the GPS/IMU localization unit~\cite{geiger2013vision}. For consecutive views in each camera, the ASIFT~\cite{morel2009asift} is used to establish the ACs. There are also strategies to speed up the extraction of ACs, such as MSER~\cite{matas2004robust}, GPU acceleration, or approximating ACs from SIFT features~\cite{guanICCV2021minimal}. For the PC-based solvers, only the PCs derived from the ACs are used. In order to address outlier matches, all solvers are integrated into a RANSAC framework. We implement a RANSAC-like procedure to select the correct solution from multiple solutions. This involves counting their inliers, with the solution presenting the most inliers being chosen.

Table~\ref{tab:RTErrror_kitti_generalized} illustrates the rotation and translation errors of the proposed \texttt{2AC method} for \texttt{KITTI} sequences, where the performance is evaluated through the median error. The experimental results demonstrate that the overall performance of the \texttt{2AC method} surpasses that of comparative methods in nearly all cases. In addition, to evaluate computational efficiency advantages, Table~\ref{RANSACTime_generalized} illustrates the RANSAC runtime averaged across all \texttt{KITTI} sequences for each of the solvers. Reported runtimes denote the relative pose estimation via RANSAC combined with a minimal solver, comprising mainly hypothesis generation and the best candidate selection from multiple solutions by counting the number of inliers. Despite some solvers having faster computation times than the proposed \texttt{2AC method} shown in Table~\ref{tab:SolverTime_generalized}, our method exhibits superior overall efficiency to all comparative methods when integrated into the RANSAC framework. The detailed analysis is presented in the supplementary material.

To provide a visual representation of the comparative results, we display the estimated trajectory for \texttt{KITTI} sequence 00. Figure~\ref{fig:trajectory} shows both the proposed \texttt{2AC method}'s estimated trajectory and the trajectory of the best-performing comparison method, \texttt{8PC-Kneip}~\cite{kneip2014efficient}, as shown in Table~\ref{tab:RTErrror_kitti_generalized}. It should be noted that the frame-to-frame relative pose estimation results are directly concatenated without any post-refinement. Both estimated trajectories are aligned with the ground truth, and their trajectories on the X-Z plane are illustrated in Fig.~\ref{fig:trajectory}. It is important to mention that our \texttt{2AC method} has a smaller Y-axis error than the \texttt{8PC-Kneip} method. Additionally, the color along the estimated trajectory signifies the absolute trajectory error (ATE)~\cite{sturm2012benchmark}, which highlights that the proposed \texttt{2AC method} has a smaller ATE than the \texttt{8PC-Kneip} method.

\subsubsection{Experiments on nuScenes Dataset}
The performance of the solvers is also tested on the \texttt{nuScenes} dataset~\cite{Caesar_2020_CVPR}. This dataset is comprised of consecutive keyframes from 6 cameras, and this multi-camera system yields a full 360 degree field of view. For evaluation purposes, all the keyframes of Part 1 are utilized, resulting in a total of 3376 image pairs. The ground truth is established using a lidar map-based localization scheme. Similar to the experiments conducted on the \texttt{KITTI} dataset, the ASIFT detector is employed to establish the ACs between consecutive views in the six cameras. The proposed \texttt{2AC~method} is compared to state-of-the-art methods, including \texttt{17PC-Li}~\cite{li2008linear}, \texttt{8PC-Kneip}~\cite{kneip2014efficient}, \texttt{6PC-Stew{'e}nius}~\cite{henrikstewenius2005solutions}, and \texttt{6AC-Ventura}~\cite{alyousefi2020multi}. To address outlier matches of the feature correspondences, all the solvers are integrated into the RANSAC framework.

Table~\ref{tab:RTErrror_nuScenes_generalized} illustrates the rotation and translation errors of the proposed \texttt{2AC method} for Part 1 of the \texttt{nuScenes} dataset. The estimation accuracy is evaluated through the median error. The experimental results demonstrate that the proposed \texttt{2AC method} outperforms all other methods. As evidenced by the comparison with experiments on the \texttt{KITTI} dataset, this experiment also showcases that the proposed \texttt{2AC method} can be directly applied to the relative pose estimation for systems with multiple cameras.
\begin{table}[tbp]
	\caption{Rotation and translation error for \texttt{nuScenes} sequences (unit: degree).}
	\begin{center}
		\setlength{\tabcolsep}{1.0mm}{
			\scalebox{0.74}{
				\begin{tabular}{c||c|c|c|c|c}
					\hline
					\multirow{2}{*}{\footnotesize{Part}} &  
					\footnotesize{17PC-Li}\scriptsize{~\cite{li2008linear}} &  \footnotesize{8PC-Kneip}\scriptsize{~\cite{kneip2014efficient}} &  \footnotesize{6PC-Stew.}\scriptsize{~\cite{henrikstewenius2005solutions}}& \footnotesize{6AC-Vent.}\scriptsize{~\cite{alyousefi2020multi}}& \footnotesize{\textbf{2AC method}} \\
					\cline{2-6}
					& ${\varepsilon _{\bf{R}}}$\quad\ $\varepsilon_{\mathbf{t},\text{dir}}$      &  ${\varepsilon _{\bf{R}}}$\quad\ $\varepsilon_{\mathbf{t},\text{dir}}$      &   ${\varepsilon _{\bf{R}}}$\quad\ $\varepsilon_{\mathbf{t},\text{dir}}$     &   ${\varepsilon _{\bf{R}}}$\quad\ $\varepsilon_{\mathbf{t},\text{dir}}$  &   ${\varepsilon _{\bf{R}}}$\quad\ $\varepsilon_{\mathbf{t},\text{dir}}$\\
					\hline
					\small{01 (3376 images)}&  0.161 \ 2.680 &  0.156 \ 2.407 & 0.203 \  2.764 & 0.143 \ 2.366 & \textbf{0.114} \ \textbf{2.017} \\
					\hline					
		\end{tabular}}}
	\end{center}
	\label{tab:RTErrror_nuScenes_generalized}
\end{table}

\subsubsection{\label{sec:EuRoCexperiments}Experiments on EuRoC Dataset}
To further test the performance of the solvers in an unmanned aerial vehicle environment, we utilize the \texttt{EuRoC} MAV dataset~\cite{burri2016euroc} to evaluate the accuracy of the 6DOF relative pose estimation. This dataset is captured using a stereo camera that is mounted on a micro aerial vehicle. The \texttt{2AC~method} is tested on all five available sequences, all of which are obtained from a large industrial machine hall. Each sequence offers synchronized stereo images, as well as accurate position and IMU measurements. The spatio-temporally aligned ground truth data is generated by means of a nonlinear least-squares batch solution over Leica position and IMU measurements. It is worth mentioning that the unstructured and cluttered nature of the industrial environment made these sequences difficult to process. To prevent the movement of the image pairs from being too small, one out of every four consecutive images is thinned out to obtain the image pairs for relative pose estimation. Furthermore, the image pairs with insufficient motion are removed from the experiment. The ACs between consecutive views within each camera are established using the ASIFT~\cite{morel2009asift}. In total, all of the solvers were tested on approximately 3000 image pairs.

Table~\ref{tab:RTErrror_EuRoC_generalized} illustrates the rotation and translation errors of the proposed \texttt{2AC method} for \texttt{EuRoC} sequences. The experimental results demonstrate that the proposed \texttt{2AC method} outperforms state-of-the-art methods \texttt{17PC-Li}, \texttt{8PC-Kneip}, \texttt{6PC-Stew{\'e}nius} and \texttt{6AC-Ventura}. This experiment effectively showcases the applicability of our \texttt{2AC method} in the context of relative pose estimation for unmanned aerial vehicles.
\begin{table}[tbp]
	\caption{Rotation and translation error for \texttt{EuRoC} sequences (unit: degree).}
	\begin{center}
		\setlength{\tabcolsep}{1.0mm}{
			\scalebox{0.72}{
				\begin{tabular}{c||c|c|c|c|c}
					\hline
					\multirow{2}{*}{\footnotesize{Seq.}} &  
					\footnotesize{17PC-Li}\scriptsize{~\cite{li2008linear}} &  \footnotesize{8PC-Kneip}\scriptsize{~\cite{kneip2014efficient}} &  \footnotesize{6PC-Stew.}\scriptsize{~\cite{henrikstewenius2005solutions}}& \footnotesize{6AC-Vent.}\scriptsize{~\cite{alyousefi2020multi}}& \footnotesize{\textbf{2AC method}} \\
					\cline{2-6}
					& ${\varepsilon _{\bf{R}}}$\quad\ $\varepsilon_{\mathbf{t},\text{dir}}$      &  ${\varepsilon _{\bf{R}}}$\quad\ $\varepsilon_{\mathbf{t},\text{dir}}$      &   ${\varepsilon _{\bf{R}}}$\quad\ $\varepsilon_{\mathbf{t},\text{dir}}$     &   ${\varepsilon _{\bf{R}}}$\quad\ $\varepsilon_{\mathbf{t},\text{dir}}$  &   ${\varepsilon _{\bf{R}}}$\quad\ $\varepsilon_{\mathbf{t},\text{dir}}$\\
					\hline
					\rowcolor{gray!10}\small{MH01 (788 images)}& 0.113 \ 2.928 &  0.109 \ 2.865& 0.124 \  3.555 & 0.106 \  2.858 &\textbf{0.092} \ \textbf{2.519}  \\
					\small{MH02 (675 images)}                  & 0.106 \ 2.494 &  0.112 \ 2.553& 0.144 \  2.908 & 0.102 \  2.483 &\textbf{0.086} \ \textbf{2.242}  \\
					\rowcolor{gray!10}\small{MH03 (605 images)}& 0.137 \ 2.412 &  0.148 \ 2.276& 0.181 \  3.068 & 0.133 \  2.075 &\textbf{0.125} \ \textbf{1.928}  \\
					\small{MH04 (449 images)}                  & 0.154 \ 2.950 &  0.170 \ 3.127& 0.175 \  5.531 & 0.165 \  2.966 &\textbf{0.139} \ \textbf{2.609}  \\
					\rowcolor{gray!10}\small{MH05 (514 images)}& 0.167 \ 3.071 &  0.158 \ 2.753& 0.179 \  4.275 & 0.176 \  2.904 &\textbf{0.146} \ \textbf{2.714}  \\
					\hline										
		\end{tabular}}}
	\end{center}
	\label{tab:RTErrror_EuRoC_generalized}
\end{table}	

\section{\label{sec:conclusion}Conclusion}

By exploiting the geometric constraints using a special parameterization, we estimate the 6DOF relative pose of a multi-camera system using a minimal number of two ACs. The extra implicit constraints about the affine transformation constraints are found and proved. We also develop the minimal solvers for 5DOF relative pose estimation of multi-camera systems with a known relative rotation angle. The proposed minimal solvers are designed to accommodate two common types of ACs across two different views, \emph{i.e.}, inter-camera and intra-camera. Moreover, three degenerate cases are proved. The framework for generating the minimal solvers is unified and versatile, and can be extended to solve various relative pose estimation problems. Compared with existing minimal solvers, our solvers require fewer feature correspondences and are not restricted to special cases of multi-camera motion. Experimental results using synthetic data and three real-world image datasets demonstrate that the proposed solvers can be used effectively for ego-motion estimation, surpassing state-of-the-art methods in both accuracy and efficiency.

\ifCLASSOPTIONcompsoc
 \section*{Acknowledgments}
\else
\fi
This work was supported in part by the Hunan Provincial Natural Science Foundation for Excellent Young Scholars under Grant 2023JJ20045 and in part by the Science Foundation under Grant KY0505072204 and Grant GJSD22006.

\ifCLASSOPTIONcaptionsoff
  \newpage
\fi



\bibliographystyle{IEEEtran}
%
{
	\bibliography{Reference}
}






\clearpage
\Large
\begin{center}
	{\bf Supplementary Material }
\end{center}
\normalsize
\input{appendix_body}


%


\end{document}

%% file: appendix_body.tex

%


\appendices

\section{\label{sec:6DOFmotion_sm}6DOF Relative Pose Estimation for Multi-Camera Systems}
In this section, we first show different equation combinations for the 6DOF relative pose estimation solvers. Then, we analyze the efficiency of the proposed solvers in a RANSAC framework.   

\subsection{Different Equation Combinations for polynomial system solving}
In subsection~\ref{sec:Solve6DOF} of the paper, we use all the equations $\mathcal{E}_{1}$ and $\mathcal{E}_2$ to construct polynomial systems and find solvers. It is possible to construct solvers using a subset of these equations.
Specifically, denote $\mathcal{E}_{1,1}$ as 
{\begin{gather}
		\det(\mathbf{N}(q_x, q_y, q_z)) = 0, \ \ 
		{\mathbf{N}\in\{3\times 3 \text{ submatrices of } {\mathbf{F}_j}\} },
\end{gather}}
and  $\mathcal{E}_{1,2}$ as 
{\begin{gather}
		\det(\mathbf{N}(q_x, q_y, q_z)) = 0, \ \ 
		{\mathbf{N}\in\{3\times 3 \text{ submatrices of } {\mathbf{F}_{j'}}\} }.
\end{gather}}

We can see that $\mathcal{E}_1 = \mathcal{E}_{1,1} \cup \mathcal{E}_{1,2}$.
Using different combinations of $\mathcal{E}_{1,1}$, $\mathcal{E}_{1,2}$, and $\mathcal{E}_2$, we have the following results for polynomial system solving. The dimension, degree, and number of solutions are shown in Table~\ref{tab:complete_solution_combination}. When the dimension of the corresponding polynomial idea is zero, it means the number of solutions is finite. Otherwise, a positive dimension of the corresponding polynomial idea indicates infinite solutions.
\begin{table}[htbp]
	\centering
	\caption{Different equation combinations for 6DOF relative pose estimation solvers. \texttt{dimension} indicates the dimension of the corresponding polynomial ideal. \texttt{degree} indicates the degree of the algebraic variety. \texttt{\#sol} indicates the number of solutions. \texttt{$1$-dim} indicates one dimensional families of extraneous roots.}
	\begin{center}
		\setlength{\tabcolsep}{2.0mm}{
			\scalebox{0.88}{
				\begin{tabular}{c||c|c|c|c|c|c} 
					\hline
					\multirow{3}{*}{\centering Equation} & \multicolumn{3}{c|}{Inter-camera} &  \multicolumn{3}{c}{Intra-camera}  \\ 
					\cline{2-7} 
					&   dimension   & degree  & \#sol  &   dimension    &   degree & \#sol  \\ 
					\hline
					{$\mathcal{E}_{1,1}$} & 1 & $2$ & 1-dim & 1 &  $3$ & 1-dim \\ \hline
					{$\mathcal{E}_{1,2}$} & 1 & $2$ & 1-dim & 1 & $3$ & 1-dim \\  \hline
					{$\mathcal{E}_2$} & 1 & $16$ & 1-dim & 1 & $16$ & 1-dim \\ \hline
					{$\mathcal{E}_1$} & 0 & $56$ & $56$ & 1 &  $1$  & 1-dim\\ \hline	{$\mathcal{E}_{1,1}$+$\mathcal{E}_2$} & 0 & $56$ & $56$ & 0 &  $56$ & $56$ \\ \hline
					{$\mathcal{E}_{1,2}$+$\mathcal{E}_2$} & 0 & $56$ & $56$ & 0 & $56$ & $56$ \\ \hline
					{$\mathcal{E}_1$+$\mathcal{E}_2$} & 0 & $48$ & $48$ & 0 & $48$ & $48$ \\ \hline
		\end{tabular}}}
	\end{center}
	\label{tab:complete_solution_combination}
	\vspace{-10pt}
\end{table}

\subsection{\label{sec:efficiencycomparison_sm}Efficiency Comparison in a RANSAC Framework}
\begin{figure}[tbp]
	\begin{center}
		\includegraphics[width=0.85\linewidth]{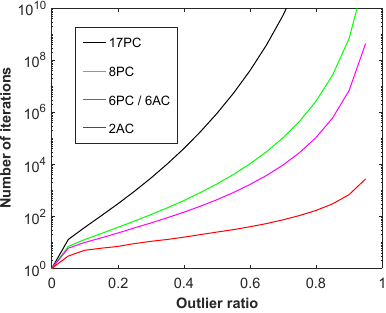}
	\end{center}
	\vspace{-10pt}
	\caption{RANSAC iteration number with respect to outlier ratio for success probability $99.9\%$. The number of RANSAC iterations increases exponentially with respect to the minimal number of feature correspondences.}
	\label{fig:RANSACIteration}
\end{figure}  
For the 6DOF relative pose estimation of multi-camera systems, we have evaluated the efficiency comparison and numerical stability of all the solvers in the paper. In addition to efficiency and numerical stability, another important factor for a solver is the minimal number of needed feature correspondences between two views. Because the minimal solvers are typically employed inside a RANSAC framework, and the computational complexity of the RANSAC estimator increases exponentially with respect to the number of feature correspondences needed. The number of iterations $N$ required in RANSAC can be given by $N=\log(1-p)/\log(1-(1-\epsilon)^s)$, where $s$ is the minimal number of feature correspondences needed for the solver, $\epsilon$ is the outlier ratio, and $p$ is the success probability that all the selected feature correspondences are inlier. For a probability of success $p=99.9\%$, the number of required RANSAC iterations with respect to the outlier ratio is shown in Fig.~\ref{fig:RANSACIteration}. 

It can be seen that the number of iterations $N$ increases exponentially with respect to the minimal number of feature correspondences $s$. For example, given the outlier ratio $\epsilon$ = $50\%$, when the solvers need $17$, $8$, $6$ and $2$ feature correspondences, the number of required RANSAC iterations is $905410$, $1765$, $439$ and $25$, respectively. Since the proposed solvers require only two ACs, the number of RANSAC iterations is obviously lower than both the PC-based methods and the AC-based linear method. Thus, our solvers have an advantage in detecting the outlier and estimating the initial motion efficiently when integrating them into the RANSAC framework. As shown in the paper, the proposed solvers have better overall efficiency than the comparative solvers in the experiments on real data.
\begin{figure*}[tbp]
	\begin{center}
		\includegraphics[width=0.75\linewidth]{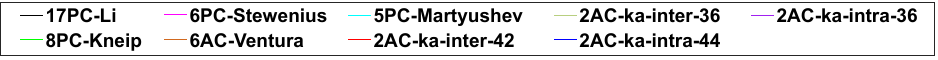}\\
		\subfigure[$\varepsilon_{\R}$]
		{
			\includegraphics[width=0.3\linewidth]{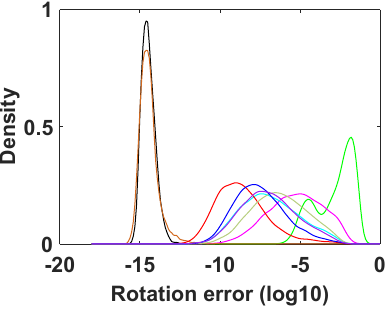}
		} 
		\subfigure[$\varepsilon_{\mathbf{t}}$]
		{
			\includegraphics[width=0.3\linewidth]{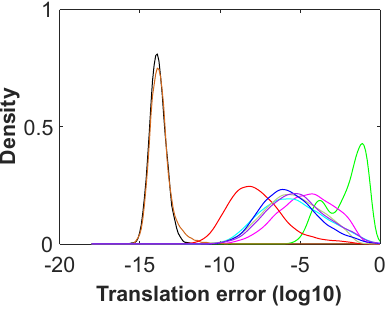}
		}
		\subfigure[$\varepsilon_{\mathbf{t},\text{dir}}$]   
		{
			\includegraphics[width=0.3\linewidth]{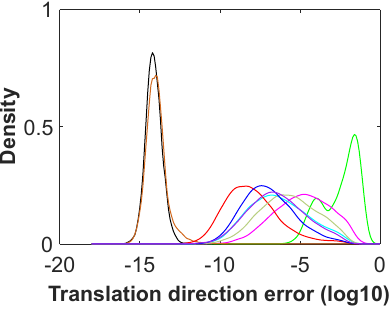}
		}
	\end{center}
	\vspace{-10pt}
	\caption{Probability density functions over 5DOF relative pose estimation errors on noise-free observations for multi-camera systems. The horizontal axis represents the $\log_{10}$ errors, and the vertical axis represents the density.}
	\label{fig:Numerical_2AC_ka}
\end{figure*} 

\begin{figure*}[tbp]
	\begin{center}
		\includegraphics[width=0.7\linewidth]{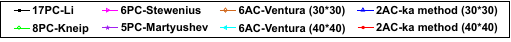}\\
		\vspace{-5pt} 
		\subfigure[${\varepsilon_{\bf{R}}}$ with image noise]
		{
			\includegraphics[width=0.303\linewidth]{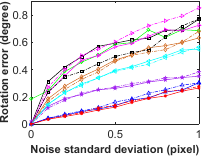}
		}
		\subfigure[${\varepsilon_{\bf{t}}}$ with image noise]
		{
			\includegraphics[width=0.303\linewidth]{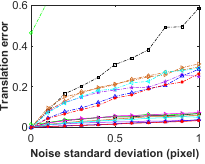}
		}
		\subfigure[$\varepsilon_{\mathbf{t},\text{dir}}$ with image noise]
		{
			\includegraphics[width=0.303\linewidth]{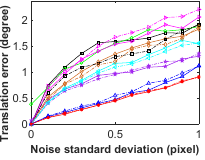}
		}
		\subfigure[${\varepsilon_{\bf{R}}}$ with rotation angle noise]
		{
			\includegraphics[width=0.303\linewidth]{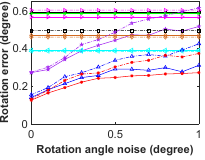}
		}
		\subfigure[${\varepsilon_{\bf{t}}}$ with rotation angle noise]
		{
			\includegraphics[width=0.303\linewidth]{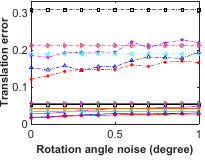}
		}
		\subfigure[$\varepsilon_{\mathbf{t},\text{dir}}$ with rotation angle noise]
		{
			\includegraphics[width=0.303\linewidth]{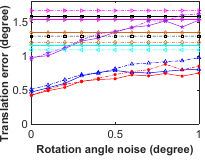}
		}
	\end{center}
	\vspace{-10pt}
	\caption{Rotation and translation error of the proposed 5DOF solvers under random motion. (a)$\sim$(c): vary image noise. (d)$\sim$(f): vary relative rotation angle noise and fix the standard deviation of image noise at 0.5 pixel. Solid line indicates using inter-camera ACs, and dash-dotted line indicates using intra-camera ACs.}
	\vspace{-5pt}
	\label{fig:RT_generalized_Pix_5DOF_Random}
\end{figure*} 

\section{\label{sec:5DOFmotion_Experiments_sm}5DOF Relative Pose Estimation for Multi-Camera Systems}
In this section, the performance of the 5DOF solvers proposed in Section~\ref{sec:5DOFmotion} is validated using both synthetic and real-world data. The configurations of the experiments are set as the same as used in Section~\ref{sec:experiments} in the paper. For 5DOF relative pose estimation with known relative rotation angle, the proposed solvers are referred to as \texttt{2AC-ka method}, specifically \texttt{2AC-ka-inter} for inter-camera ACs and \texttt{2AC-ka-intra} for intra-camera ACs. To further distinguish two solvers for inter-camera ACs, \texttt{2AC-inter-42} and \texttt{2AC-inter-36} are used to refer to the solvers resulting from $\mathcal{E}_1$ and $\mathcal{E}_1+\mathcal{E}_2$, respectively. Meanwhile, to further distinguish two solvers for intra-camera ACs, \texttt{2AC-intra-44} and \texttt{2AC-intra-36} are used to refer to the solvers resulting from $\mathcal{E}_1$ and $\mathcal{E}_1+\mathcal{E}_2$, respectively. 
\begin{figure*}[tbp]
	\begin{center}
		\includegraphics[width=0.7\linewidth]{figure/Legend_2AC5DOF_stew_ka.pdf}\\
		\subfigure[${\varepsilon_{\bf{R}}}$ with image noise]
		{
			\includegraphics[width=0.303\linewidth]{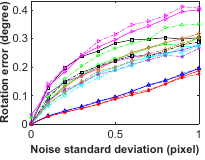}
		}
		\subfigure[${\varepsilon_{\bf{t}}}$ with image noise]
		{
			\includegraphics[width=0.303\linewidth]{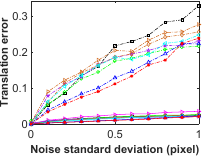}
		}
		\subfigure[$\varepsilon_{\mathbf{t},\text{dir}}$ with image noise]
		{
			\includegraphics[width=0.303\linewidth]{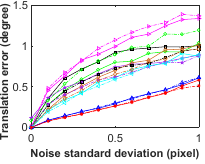}
		}
		\subfigure[${\varepsilon_{\bf{R}}}$ with rotation angle noise]
		{
			\includegraphics[width=0.303\linewidth]{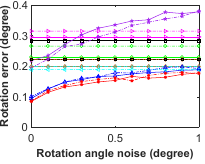}
		}
		\subfigure[${\varepsilon_{\bf{t}}}$ with rotation angle noise]
		{
			\includegraphics[width=0.303\linewidth]{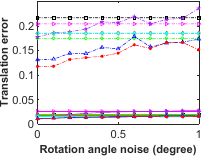}
		}
		\subfigure[$\varepsilon_{\mathbf{t},\text{dir}}$ with rotation angle noise]
		{
			\includegraphics[width=0.303\linewidth]{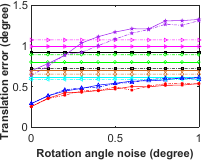}
		}
	\end{center}
	\caption{Rotation and translation error of the proposed 5DOF solvers under forward motion. (a)$\sim$(c): vary image noise. (d)$\sim$(f): vary relative rotation angle noise and fix the standard deviation of image noise at 1.0 pixel. Solid line indicates using inter-camera ACs, and dash-dotted line indicates using intra-camera ACs.}
	\label{fig:RT_generalized_Pix_5DOF_Forward}
\end{figure*} 

\begin{figure*}[tbp]
	\begin{center}
		\includegraphics[width=0.7\linewidth]{figure/Legend_2AC5DOF_stew_ka.pdf}\\
		\subfigure[${\varepsilon_{\bf{R}}}$ with image noise]
		{
			\includegraphics[width=0.303\linewidth]{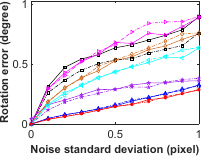}
		}
		\subfigure[${\varepsilon_{\bf{t}}}$ with image noise]
		{
			\includegraphics[width=0.303\linewidth]{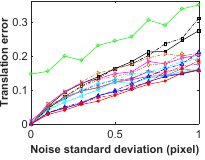}
		}
		\subfigure[$\varepsilon_{\mathbf{t},\text{dir}}$ with image noise]
		{
			\includegraphics[width=0.303\linewidth]{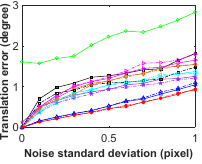}
		}
		\subfigure[${\varepsilon_{\bf{R}}}$ with rotation angle noise]
		{
			\includegraphics[width=0.303\linewidth]{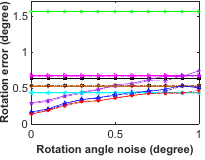}
		}
		\subfigure[${\varepsilon_{\bf{t}}}$ with rotation angle noise]
		{
			\includegraphics[width=0.303\linewidth]{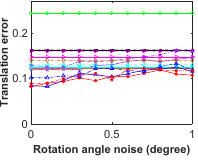}
		}
		\subfigure[$\varepsilon_{\mathbf{t},\text{dir}}$ with rotation angle noise]
		{
			\includegraphics[width=0.303\linewidth]{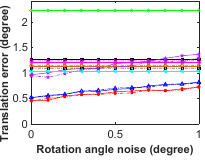}
		}
	\end{center}
	\caption{Rotation and translation error of the proposed 5DOF solvers under sideways motion. (a)$\sim$(c): 
		vary image noise. (d)$\sim$(f): vary relative rotation angle noise and fix the standard deviation of image noise at 1.0 pixel. Solid line indicates using inter-camera ACs, and dash-dotted line indicates using intra-camera ACs.}
	\label{fig:RT_generalized_Pix_5DOF_Sideways}
\end{figure*} 

The \texttt{2AC-ka} solvers are compared with state-of-the-art methods including \texttt{17PC-Li}~\cite{li2008linear}, \texttt{8PC-Kneip}~\cite{kneip2014efficient}, \texttt{6PC-Stew{\'e}nius}~\cite{henrikstewenius2005solutions}, \texttt{6AC-Ventura}~\cite{alyousefi2020multi} and \texttt{5PC-Martyushev}~\cite{martyushev2020efficient}. Note that the \texttt{5PC-Martyushev} solver also uses the known relative rotation angle as a prior. The proposed solvers are not compared with the methods which estimate the relative pose with different motion prior, such as planar motion~\cite{guanICCV2021minimal}, known gravity direction~\cite{sweeney2014solving,guanICCV2021minimal} and first-order approximation assumption~\cite{ventura2015efficient,Guan_ICRA2021}. All the solvers are integrated into RANSAC in order to remove outlier matches. The relative pose which produces the most inliers is used to measure the relative pose error. This also allows us to select the best candidate from multiple solutions by counting their inliers. 

\subsection{Numerical Accuracy}
Figure~\ref{fig:Numerical_2AC_ka} reports the numerical stability comparison of the 5DOF solvers in noise-free cases. We repeat the procedure 10,000 times and plot the empirical probability density functions as the function of the $\log_{10}$ estimated errors. Numerical stability represents the round-off error of solvers in noise-free cases. The solvers \texttt{17PC-Li}~\cite{li2008linear} and \texttt{6AC-Ventura}~\cite{alyousefi2020multi} have the best numerical stability, the linear solvers with smaller computation burden have less round-off error. Since the \texttt{8PC-Kneip}~\cite{kneip2014efficient} uses the iterative optimization, it is susceptible to falling into local minima. The \texttt{2AC-ka-inter-42} and \texttt{2AC-ka-intra-44} have better numerical stability than the \texttt{2AC-ka-inter-36} and \texttt{2AC-ka-intra-36}, respectively. It is shown that adding the extra equations $\mathcal{E}_2$ is not helpful in terms of numerical stability. This phenomenon has also been observed in previous literature~\cite{byrod2009fast}, which shows that the number of basis affect the numerical stability. Even though \texttt{2AC-ka-inter-36} and \texttt{2AC-ka-intra-36} produce fewer solutions, we prefer to perform \texttt{2AC-ka-inter-42} and \texttt{2AC-ka-intra-44} for the sake of numerical accuracy in the follow-up experiments.  

\subsection{Experiments on Synthetic Data}
In the experiments, the required ACs are selected randomly for the solvers within the RANSAC scheme. For the PC-based solvers, only the PC of the AC is used. The motion directions of the multi-camera system are set to random, forward, and sideways motions, respectively.

\subsubsection{Accuracy with Image Noise}
In this scenario, the magnitude of image noise is set to Gaussian noise with a standard deviation ranging from 0 to 1 pixel. Figure~\ref{fig:RT_generalized_Pix_5DOF_Random}(a)$\sim$(c) show the performance of the proposed 5DOF solvers with increasing image noise under random motion. All the solvers are evaluated on both inter-camera ACs and intra-camera ACs. The corresponding estimation results are represented by solid lines and dash-dotted lines, respectively. The \texttt{2AC-ka} indicates \texttt{2AC-ka-inter-42} when using inter-camera ACs, and indicates \texttt{2AC-ka-intra-44} when using intra-camera ACs. The performance of AC-based methods is influenced by the noise magnitude of affine transformation, which is determined by the support region of sampled points. The display range is limited so that some curves with large errors are invisible or partially invisible. The proposed 5DOF solvers using inter-camera ACs generally have better performance than intra-camera ACs, especially in recovering the metric scale of translation. The \texttt{8PC-Kneip} performs poorly in sideways motion, and the probable reason may be the iterative optimization which is susceptible to falling into local minima~\cite{ventura2015efficient}. Our \texttt{2AC-ka} provides better results than the comparative methods with both inter-camera ACs and intra-camera ACs, even though the side length of the square is $20$~pixels. 

Figure~\ref{fig:RT_generalized_Pix_5DOF_Forward}(a)$\sim$(c) show the performance of the proposed 5DOF solvers with increasing image noise under forward motion. Figure~\ref{fig:RT_generalized_Pix_5DOF_Sideways}(a)$\sim$(c) show the performance of the proposed 5DOF solvers with increasing image noise under sideways motion. Solid lines and dash-dotted lines represent the evaluation results with inter-camera ACs and intra-camera ACs, respectively. The \texttt{2AC-ka} indicates \texttt{2AC-ka-inter-42} when using inter-camera ACs, and indicates \texttt{2AC-ka-intra-44} when using intra-camera ACs. The \texttt{2AC-ka} provides better results than the comparative methods with both inter-camera ACs and intra-camera ACs, even though the side length of the square is $20$~pixels. 
	
\subsubsection{Accuracy with Rotation Angle Noise}
\begin{figure}[htbp]
	\begin{center}
		\includegraphics[width=1.0\linewidth]{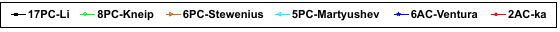}\\
		\subfigure[\scriptsize{${\varepsilon _{\bf{R}}}$ with image noise}]
		{
			\includegraphics[width=0.96\linewidth]{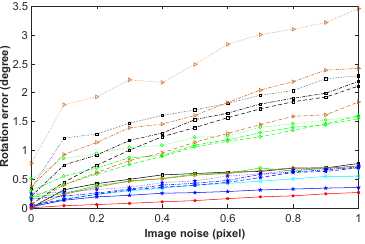}
		}\\ 
		\subfigure[\scriptsize{${\varepsilon _{\bf{t}}}$ with image noise}]
		{
			\includegraphics[width=0.965\linewidth]{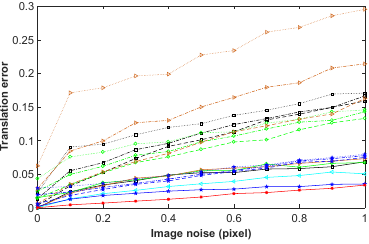}
		}\\
		\subfigure[\scriptsize{$\varepsilon_{\mathbf{t},\text{dir}}$ with image noise}]
		{
			\includegraphics[width=0.96\linewidth]{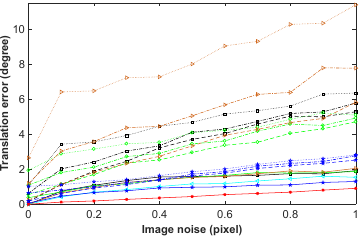}
		}
	\end{center}
	\vspace{-10pt} 
	\caption{Rotation and translation error of the PC-based solvers using ACs with varying image noise under random motion. Solid line indicates using image point pairs of ACs. Dashed line, dash-dotted line, and dotted line indicate using the hallucinated PCs, which are generated with different distribution areas $s$ = 1, 5, and 10 pixels, respectively.}		
	\vspace{-10pt} 
	\label{fig:RTRandomMotion_1ACto3PCs_5DOF}
\end{figure}

In this scenario, the image noise is set to 0.5 pixel standard deviation. The magnitude of rotation angle noise is set to Gaussian noise with a standard deviation ranging from $0^\circ$ to $1^\circ$. Figure~\ref{fig:RT_generalized_Pix_5DOF_Random}(d)$\sim$(f) show the performance of the proposed 5DOF solvers with increasing rotation angle noise under random motion. The methods \texttt{17PC-Li}~\cite{li2008linear}, \texttt{8PC-Kneip}~\cite{kneip2014efficient}, \texttt{6PC-Stew{\'e}nius}~\cite{henrikstewenius2005solutions}, and \texttt{6AC-Ventura}~\cite{alyousefi2020multi} are not influenced by the rotation angle noise, because their calculation does not utilize the known relative rotation angle as prior and estimate 5DOF relative pose of two views directly. In comparison with using intra-camera ACs, the solvers using inter-camera ACs perform better, especially in estimating the metric scale of translation. Due to the large error, the results of \texttt{8PC-Kneip} using intra-camera ACs are invisible in this figure. It is seen that \texttt{2AC-ka} is robust to the increasing rotation angle noise and provides better results than the comparative methods. 

Figure~\ref{fig:RT_generalized_Pix_5DOF_Forward}(d)$\sim$(f) show the performance of the proposed 5DOF solvers with increasing rotation angle noise under forward motion. 	Figure~\ref{fig:RT_generalized_Pix_5DOF_Sideways}(d)$\sim$(f) show the performance of the proposed 5DOF solvers with increasing rotation angle noise under sideways motion. The solvers using inter-camera ACs have better performance than the solvers using intra-camera ACs, especially in recovering the metric scale of translation. It can be seen that the proposed \texttt{2AC-ka} is robust to the increasing rotation angle noise and provides better results than the comparative methods.

\subsubsection{{Evaluation of PC-based Solvers using ACs}}
Similar to Section~\ref{sec:generatedPCs_sm} in the paper, we test the performance of PC-based solvers for the 5DOF relative pose estimation using ACs. Take 5DOF relative pose estimation using inter-camera ACs for an example. The synthetic data is generated by following the configuration in Section~\ref{sec:generatedPCs_sm}. A total of 1000 trials are carried out in the synthetic experiment. The rotation and translation errors are assessed by the median of errors. In each test, 100 inter-camera ACs are generated randomly with $40\times40$ support region. In the RANSAC loop, 6 and 2 ACs are selected randomly for the \texttt{6AC-Ventura}~\cite{alyousefi2020multi} method and the proposed \texttt{2AC-ka-inter-42} method, respectively. The hallucinated PCs converted from a minimal number of ACs are used as input for the PC-based solvers. Thus, 6, 3, and 2 ACs are selected randomly for \texttt{17PC-Li}~\cite{li2008linear}, \texttt{8PC-Kneip}~\cite{kneip2014efficient}, and the solvers \texttt{6PC-Stew{\'e}nius}~\cite{henrikstewenius2005solutions} and \texttt{5PC-Martyushev}~\cite{martyushev2020efficient}, respectively. Note that the hallucinated PCs converted from ACs are only used for hypothesis generation, and the inlier set is found by evaluating the image point pairs of ACs. The solution which produces the highest number of inliers is chosen. The other configurations of this synthetic experiment are set as the same as used in Figure~\ref{fig:RT_generalized_Pix_5DOF_Random}(a)$\sim$(c) in the paper. 
	
Figure~\ref{fig:RTRandomMotion_1ACto3PCs_5DOF} shows the performance of the PC-based solvers with increasing image noise under random motion. The ego-motion estimation results using the image point pairs of ACs are represented by solid lines. The ego-motion estimation results using the hallucinated PCs generated with different distribution areas are represented by dashed line ($s$ = 1 pixel), dash-dotted line ($s$ = 5 pixels), and dotted line ($s$ = 10 pixels), respectively. We have the following observations. (1) The PC-based solvers using the hallucinated PCs perform worse than those using the image point pairs of AC. Because the conversion error between each AC and three PCs is newly introduced. It can be seen that the relative pose estimation error of PC-based solvers using the hallucinated PCs is not zero, even for image noise-free input. Moreover, the hallucinated PCs generated by each AC are near each other, which may be a degenerate case for the PC-based solvers. (2) The performance of PC-based solvers is influenced by the different distribution areas of hallucinated PCs. Since a smaller distribution area causes smaller conversion errors between ACs and PCs, the PC-based solvers have better performance with a smaller distribution area. (3) The performance of the proposed \texttt{2AC-ka} is best. Because the AC-based solvers use the relationship between local affine transformations and epipolar lines (Eq.~\eqref{eq:constraint_affine_single} in the paper). This is a strictly satisfied constraint and does not result in any error for noise-free input.

\subsection{\label{sec:realexperiments}Experiments on Real Data}
The performance of the proposed 5DOF solvers is evaluated on three public datasets in popular modern applications of mobile devices. Specifically, two datasets \texttt{KITTI} ~\cite{geiger2013vision} and \texttt{nuScenes}~\cite{Caesar_2020_CVPR} are collected using outdoor autonomous vehicles. One dataset \texttt{EuRoc}~\cite{burri2016euroc} is collected using an unmanned aerial vehicle. There are a total of 30,000 image pairs in three datasets. A lot of challenging image pairs including large motion and highly dynamic scenes are provided in these datasets. The proposed solvers are compared against state-of-the-art relative pose estimation methods.
	
\subsubsection{Experiments on KITTI Dataset}
Our solvers are evaluated on \texttt{KITTI} dataset~\cite{geiger2013vision} collected on outdoor autonomous vehicles with forward facing cameras. We treat it as a general multi-camera system by ignoring the overlap in their fields of view. The \texttt{2AC-ka-intra} solver is tested on all the available 11 sequences, which consists of 23000 image pairs in total. The ground truth is directly given by the output of the GPS/IMU localization unit~\cite{geiger2013vision}. The relative rotation angle is obtained from the fused GPS/IMU pose~\cite{Li2020Relative}. For consecutive views in each camera, the ASIFT~\cite{morel2009asift} is used to establish the ACs. Note that only the PCs of the ACs are used for the PC-based solvers. To deal with outlier matches, all the solvers are integrated into the RANSAC. To select the right solution from multiple solutions, we counted their inliers in a RANSAC-like procedure and the solution with the most inliers is chosen. 
	
Table~\ref{tab:RTErrror_kitti_generalized_5DOF} shows the rotation and translation error of the proposed 5DOF solvers for \texttt{KITTI} sequences. The median error is used to evaluate the performance. The overall performance of the \texttt{2AC-ka} is best. Table~\ref{RANSACTime_generalized_5DOF} shows the corresponding RANSAC runtime averaged over all the KITTI sequences for the proposed 5DOF solvers. The reported runtimes include the relative pose estimation by RANSAC combined with a minimal solver. Since the proposed solvers require fewer correspondences, they have the advantage of computational efficiency when integrating them into RANSAC. Even though the proposed solvers are not the fastest, they need fewer iterations in the RANSAC framework. The overall efficiency is better than or comparable with the state-of-the-art methods. Based on the experiments on the \texttt{KITTI} dataset, we demonstrate that our solvers can be used efficiently for ego-motion estimation and outperforms the state-of-the-art methods in both accuracy and efficiency.
\begin{table}[tbp]
	\caption{Rotation and translation error of the proposed 5DOF solvers on \texttt{KITTI} sequences (unit: degree).}
		\vspace{-5pt}
		\begin{center}
			\setlength{\tabcolsep}{1.0mm}{
				\scalebox{0.76}{
					\begin{tabular}{c|c|c|c|c|c|c}
						\hline
						\multirow{2}{*}{\footnotesize{Seq.}} &  
						\footnotesize{17PC-Li}\scriptsize{~\cite{li2008linear}} &  \footnotesize{8PC-Kneip}\scriptsize{~\cite{kneip2014efficient}} &  \footnotesize{6PC-Stew.}\scriptsize{~\cite{henrikstewenius2005solutions}}  &  \footnotesize{6AC-Vent.}\scriptsize{~\cite{alyousefi2020multi}} 	&  \footnotesize{5PC-Mart.}\scriptsize{~\cite{martyushev2020efficient}}& \footnotesize{\textbf{2AC-ka}} \\
						\cline{2-7}
						& ${\varepsilon _{\bf{R}}}$\quad\ $\varepsilon_{\mathbf{t},\text{dir}}$    &  ${\varepsilon _{\bf{R}}}$\quad\ $\varepsilon_{\mathbf{t},\text{dir}}$  &   ${\varepsilon _{\bf{R}}}$\quad\ $\varepsilon_{\mathbf{t},\text{dir}}$ 	& ${\varepsilon _{\bf{R}}}$\quad\ $\varepsilon_{\mathbf{t},\text{dir}}$    &   ${\varepsilon _{\bf{R}}}$\quad\ $\varepsilon_{\mathbf{t},\text{dir}}$  &   ${\varepsilon _{\bf{R}}}$\quad\ $\varepsilon_{\mathbf{t},\text{dir}}$\\
						\hline
						00&                   0.139 \ 2.412 &  0.130  \  2.400& 0.229 \ 4.007 & 0.142 \ 2.499 & 0.123  \  2.485 &\textbf{0.117} \ \textbf{1.908}     \\
						01&                   0.158 \ 5.231 &  0.171  \  4.102& 0.762 \ 41.19 & 0.146 \ 3.654 & 0.096  \  4.539 &\textbf{0.077} \ \textbf{1.649}     \\
						02&                   0.123 \ 1.740 &  0.126  \  1.739& 0.186 \ 2.508 & 0.121 \ 1.702 & 0.117  \  1.832 &\textbf{0.109} \ \textbf{1.623}     \\
						03&                   0.115 \ 2.744 &  0.108  \  2.805& 0.265 \ 6.191 & 0.113 \ 2.731 & 0.094  \  3.468 &\textbf{0.081} \ \textbf{2.379}     \\
						04&                   0.099 \ 1.560 &  0.116  \  1.746& 0.202 \ 3.619 & 0.100 \ 1.725 & 0.077  \  1.717 &\textbf{0.063} \ \textbf{1.465}     \\
						05&                   0.119 \ 2.289 &  0.112  \  2.281& 0.199 \ 4.155 & 0.116 \ 2.273 & 0.106  \  2.374 &\textbf{0.079} \ \textbf{1.681}     \\
						06&                   0.116 \ 2.071 &  0.118  \  1.862& 0.168 \ 2.739 & 0.115 \ 1.956 & 0.105  \  1.839 &\textbf{0.084} \ \textbf{1.503}     \\
						07&                   0.119 \ 3.002 &  0.112  \  3.029& 0.245 \ 6.397 & 0.137 \ 2.892 & 0.108  \  2.784 &\textbf{0.098} \ \textbf{1.997}     \\
						08&                   0.116 \ 2.386 &  0.111  \  2.349& 0.196 \ 3.909 & 0.108 \ 2.344 & 0.109  \  2.505 &\textbf{0.086} \ \textbf{2.172}     \\
						09&                   0.133 \ 1.977 &  0.125  \  1.806& 0.179 \ 2.592 & 0.124 \ 1.876 & 0.112  \  1.953 &\textbf{0.108} \ \textbf{1.456}     \\
						10&                   0.127 \ 1.889 &  0.115  \  1.893& 0.201 \ 2.781 & 0.203 \ 2.057 & \textbf{0.106}  \  1.826 &0.170 \ \textbf{1.601}      \\
						\hline					
			\end{tabular}}}
		\end{center}
		\vspace{-1pt}
		\label{tab:RTErrror_kitti_generalized_5DOF}
\end{table}

\begin{table}[tbp]
		\caption{Runtime of RANSAC averaged over \texttt{KITTI} sequences combined with the proposed 5DOF solvers (unit:~$s$).}
		\vspace{-5pt}
		\begin{center}
			\setlength{\tabcolsep}{1.0mm}{
				\scalebox{0.65}{
					\begin{tabular}{c|c|c|c|c|c|c}
						\hline
						\small{Methods} &  \small{17PC-Li}\scriptsize{~\cite{li2008linear}} &  \small{8PC-Kneip}\scriptsize{~\cite{kneip2014efficient}} &  \small{6PC-St.}\scriptsize{~\cite{henrikstewenius2005solutions}}&
						\small{6AC-Ven.}\scriptsize{~\cite{alyousefi2020multi}}&
						\small{5PC-Mar.}\scriptsize{~\cite{martyushev2020efficient}}& \small{\textbf{2AC-ka}} \\
						\hline
						\small{Mean time }& 52.82 & 10.36 & 79.76& 6.83& 5.97& \textbf{3.89}\\
						\hline
						\small{Standard deviation}& 2.62 & 1.59 & 4.52& 0.61& 0.48& \textbf{0.27}\\
						\hline
			\end{tabular}}}
		\end{center}
		\label{RANSACTime_generalized_5DOF}
\end{table}
		
\subsubsection{Experiments on nuScenes Dataset}
We also test the performance of the proposed solvers on \texttt{nuScenes} dataset~\cite{geiger2013vision}. The multi-camera system is composed of six perspective cameras and provides full 360 degree field of view. The proposed solvers are applied to every consecutive image pair of Part 1. In total, 3376 consecutive image pairs are used from this dataset. The ground truth pose is obtained from a lidar map-based localization scheme. We establish the ACs between consecutive views in each camera by using the ASIFT detector~\cite{morel2009asift}. The \texttt{2AC-ka} solver is compared against the state-of-the-art methods including \texttt{17PC-Li}~\cite{li2008linear}, \texttt{8PC-Kneip}~\cite{kneip2014efficient}, \texttt{6PC-Stew{\'e}nius}~\cite{henrikstewenius2005solutions}, \texttt{6AC-Ventura}~\cite{alyousefi2020multi} and \texttt{5PC-Martyushev}~\cite{martyushev2020efficient}. In order to remove outlier matches of the feature correspondences, all the solvers are used within RANSAC.
	
Table~\ref{tab:RTErrror_nuScenes_generalized_5DOF} shows the rotation and translation error of the proposed \texttt{2AC-ka} solver for the \texttt{nuScenes} dataset. The median of errors is used to assess the rotation and translation errors. It can be seen that the proposed solvers outperform comparable state-of-the-art methods in the ego-motion of multi-camera systems.   
\begin{table}[tbp]
		\caption{Rotation and translation error of the proposed 5DOF solvers on \texttt{nuScenes} sequences (unit: degree).}
		\vspace{-5pt}
		\begin{center}
			\setlength{\tabcolsep}{1.0mm}{
				\scalebox{0.76}{
					\begin{tabular}{c|c|c|c|c|c|c}
						\hline
						\multirow{2}{*}{\footnotesize{Seq.}} &  
						\footnotesize{17PC-Li}\scriptsize{~\cite{li2008linear}} &  \footnotesize{8PC-Kneip}\scriptsize{~\cite{kneip2014efficient}} &  \footnotesize{6PC-Stew.}\scriptsize{~\cite{henrikstewenius2005solutions}}  &  \footnotesize{6AC-Vent.}\scriptsize{~\cite{alyousefi2020multi}} 	&  \footnotesize{5PC-Mart.}\scriptsize{~\cite{martyushev2020efficient}}& \footnotesize{\textbf{2AC-ka}} \\
						\cline{2-7}
						& ${\varepsilon _{\bf{R}}}$\quad\ $\varepsilon_{\mathbf{t},\text{dir}}$    &  ${\varepsilon _{\bf{R}}}$\quad\ $\varepsilon_{\mathbf{t},\text{dir}}$  &   ${\varepsilon _{\bf{R}}}$\quad\ $\varepsilon_{\mathbf{t},\text{dir}}$ 	& ${\varepsilon _{\bf{R}}}$\quad\ $\varepsilon_{\mathbf{t},\text{dir}}$    &   ${\varepsilon _{\bf{R}}}$\quad\ $\varepsilon_{\mathbf{t},\text{dir}}$  &   ${\varepsilon _{\bf{R}}}$\quad\ $\varepsilon_{\mathbf{t},\text{dir}}$\\
						\hline
						01&  0.161 \ 2.680 &  0.156 \ 2.407 & 0.203 \  2.764 & 0.143 \ 2.366 & 0.103 \ 2.119 & \textbf{0.095} \ \textbf{1.872} \\
						\hline					
			\end{tabular}}}
		\end{center}
		\label{tab:RTErrror_nuScenes_generalized_5DOF}
\end{table}
	
\subsubsection{Experiments on EuRoC Dataset}
To further illustrate the usefulness of the proposed solvers in the unmanned aerial vehicle environment, we test them on the \texttt{EuRoC} MAV dataset~\cite{burri2016euroc}. The consecutive image pairs are collected using a stereo camera mounted on a micro aerial vehicle. The proposed solvers are evaluated on all the available five sequences, which are recorded in a large industrial machine hall. The environment is unstructured and cluttered, which renders the sequences challenging to process. The ground truth pose is obtained from the nonlinear least-squares batch solution over the IMU measurements and the accurate position provided by a Leica laser tracker. The relative rotation angle is derived from the IMU measurements~\cite{martyushev2020efficient}. To prevent the movement of image pairs from being too small, we reduce the number of image pairs by taking one out of every four consecutive images in the sequences. Moreover, we also crop the image pairs with insufficient motion. In total, 3000 image pairs are used from this dataset. We also establish the ACs between consecutive views in each camera using the ASIFT detector~\cite{morel2009asift}. All the solvers are used within RANSAC to remove outlier matches. 
\begin{table}[tbp]
		\caption{Rotation and translation error of the proposed 5DOF solvers on \texttt{EuRoC} sequences (unit: degree).}
		\vspace{-5pt}
		\begin{center}
			\setlength{\tabcolsep}{1.0mm}{
				\scalebox{0.76}{
					\begin{tabular}{c|c|c|c|c|c|c}
						\hline
						\multirow{2}{*}{\footnotesize{Seq.}} &  
						\footnotesize{17PC-Li}\scriptsize{~\cite{li2008linear}} &  \footnotesize{8PC-Kneip}\scriptsize{~\cite{kneip2014efficient}} &  \footnotesize{6PC-Stew.}\scriptsize{~\cite{henrikstewenius2005solutions}}  &  \footnotesize{6AC-Vent.}\scriptsize{~\cite{alyousefi2020multi}} 	&  \footnotesize{5PC-Mart.}\scriptsize{~\cite{martyushev2020efficient}}& \footnotesize{\textbf{2AC-ka}} \\
						\cline{2-7}
						& ${\varepsilon _{\bf{R}}}$\quad\ $\varepsilon_{\mathbf{t},\text{dir}}$    &  ${\varepsilon _{\bf{R}}}$\quad\ $\varepsilon_{\mathbf{t},\text{dir}}$  &   ${\varepsilon _{\bf{R}}}$\quad\ $\varepsilon_{\mathbf{t},\text{dir}}$ 	& ${\varepsilon _{\bf{R}}}$\quad\ $\varepsilon_{\mathbf{t},\text{dir}}$    &   ${\varepsilon _{\bf{R}}}$\quad\ $\varepsilon_{\mathbf{t},\text{dir}}$  &   ${\varepsilon _{\bf{R}}}$\quad\ $\varepsilon_{\mathbf{t},\text{dir}}$\\
						\hline
						{01}& 0.113 \ 2.928 &  0.109 \ 2.865& 0.124 \  3.555 & 0.106 \  2.858 &{0.085} \ {2.286}  & \textbf{0.071} \ \textbf{2.132}  \\
						{02}& 0.106 \ 2.494 &  0.112 \ 2.553& 0.144 \  2.908 & 0.102 \  2.483 &{0.093} \ {2.335}  & \textbf{0.065} \ \textbf{1.916}  \\
						{03}& 0.137 \ 2.412 &  0.148 \ 2.276& 0.181 \  3.068 & 0.133 \  2.075 &{0.109} \ {2.136}  & \textbf{0.103} \ \textbf{1.824}  \\
						{04}& 0.154 \ 2.950 &  0.170 \ 3.127& 0.175 \  5.531 & 0.165 \  2.966 &{0.114} \ {2.476}  & \textbf{0.105} \ \textbf{2.330}  \\
						{05}& 0.167 \ 3.071 &  0.158 \ 2.753& 0.179 \  4.275 & 0.176 \  2.904 &{0.096} \ {2.518}  & \textbf{0.087} \ \textbf{2.216}  \\
						\hline							
			\end{tabular}}}
		\end{center}
		\label{tab:RTErrror_EuRoC_generalized_5DOF}
\end{table}
	
Table~\ref{tab:RTErrror_EuRoC_generalized_5DOF} shows the rotation and translation error of the proposed \texttt{2AC-ka} solver for the \texttt{EuRoC} dataset. The median of errors is used to assess the rotation and translation errors. These experiment results show that the proposed solvers provide better results than the comparable state-of-the-art solvers. Due to the benefits of computational efficiency, the proposed \texttt{2AC-ka} solver is suitable for performing outlier removal and initial motion estimation in the unmanned aerial vehicle environment. 